\documentclass{article} 
\usepackage{iclr2027_conference,times}

\usepackage{amsmath,amsfonts,bm}

\def\eqref#1{equation~\ref{#1}}

\def\1{\bm{1}}

\DeclareMathAlphabet{\mathsfit}{\encodingdefault}{\sfdefault}{m}{sl}
\SetMathAlphabet{\mathsfit}{bold}{\encodingdefault}{\sfdefault}{bx}{n}

\newcommand{\E}{\mathbb{E}}

\newcommand{\KL}{D_{\mathrm{KL}}}

\DeclareMathOperator*{\argmax}{arg\,max}

\usepackage{hyperref}
\usepackage{url}            
\usepackage{booktabs}       
\usepackage{amsfonts}       
\usepackage{nicefrac}       
\usepackage{microtype}      
\usepackage{xcolor}         

\usepackage{mathrsfs}
\usepackage{amsmath}
\usepackage{amssymb}
\usepackage{mathtools}
\usepackage{amsthm}
\usepackage{bbm}
\usepackage{url}
\usepackage{hyperref}
\hypersetup{
    colorlinks,
    linkcolor={black!50!black},
    citecolor={blue!50!black},
    urlcolor={blue!80!black}
}

\usepackage{dblfloatfix}

\usepackage{algorithm}
\usepackage{algpseudocode}
\usepackage{etoc}

\newtheorem{theorem}{Theorem}
\newtheorem{definition}{Definition}
\newtheorem{lemma}{Lemma}

\newtheorem{remark}{Remark}

\def\S{\mathcal{S}}
\def\A{\mathcal{A}}

\def\M{\mathcal{M}}
\def\I{\mathcal I}

\def\N{\mathcal N}
\def\Q{\mathcal Q}
\def\T{\mathcal T}

\def\A{\mathcal A}
\def\E{\mathbb E}
\def\V{\mathbb V}

\definecolor{claudegreen}{rgb}{0.00,0.45,0.10}

\newenvironment{clblock}{\color{claudegreen}}{}

\title{Towards Optimal Policy Improvement}

\author{
  Yaniv Oren$^{1,5}$ 
  \And Viliam Vadocz$^2$ 
  \And Wiktor Zabka$^1$ 
  \And Thomas Evers$^1$ 
  \And Jan Robine$^1$ 
  \And Wendelin B{\"o}hmer$^1$ 
  \And Matthijs T. J. Spaan$^1$ 
  \And Martha White$^3$ 
  \And Hendrik Baier$^{4,5}$ 
  \And Fenghui Yu$^6$
  \And 
  \\
  $^1$Department of Intelligent Systems, TU Delft\\
  $^2$Department of Computer Science, ETH Zürich\\
  $^3$Department of Computing Science, University of Alberta\\
  $^4$Information Systems, TU Eindhoven\\
  $^5$Centrum Wiskunde \& Informatica, Amsterdam\\
  $^6$Delft Institute of Applied Mathematics, TU Delft\\
  \texttt{\{y.oren,j.w.bohmer,m.t.j.spaan,fenghui.yu\}@tudelft.nl}\\
  \texttt{vvadocz@student.ethz.ch}\\
  \texttt{whitem@ualberta.ca}\\
  \texttt{h.j.s.baier@tue.nl}\\
}

\iclrfinalcopy 
\begin{document}
\etocdepthtag.toc{main}

\maketitle
\lhead{Preprint}

\begin{abstract}
Practical Reinforcement Learning (RL) algorithms learn to solve Markov Decision Processes (MDPs) through iterative policy improvement in the presence of approximate evaluation. 
We study policy improvement from first principles, defining optimal policy improvement as producing the best policy attainable in a single update under specified constraints.
We show that optimal improvement restricted to a set of states is equivalent to solving an induced MDP, characterizing planning with an explicit or implicit model as a path towards optimal policy improvement.
Because practical methods commonly solve such induced problems through iterative improvement in the form of greedification, we take steps towards optimal greedification under the central practical constraint of approximate evaluation. 
We formulate greedification under this constraint as probabilistic decision-making under uncertainty and derive a novel operator that is optimal with respect to the resulting objective.
Empirically, the operator and its practical gradient-based approximations improve aggregate performance across GumbelAlphaZero, SAC, ReBRAC and Generalized Policy Iteration, 
in experiments spanning discrete and continuous actions, model-based and
model-free, online and offline RL.
\end{abstract}

\section{Introduction}
\label{sec:intro}
The objective of Reinforcement Learning (RL) agents is to find well-performing policies for sequential decision problems, formalized as \textit{Markov Decision Processes} \citep[MDPs,][]{bellman1957markovian}.
The performance of the policy is measured as its \textit{expected return}, also called the policy's \textit{value}.
To find such policies, practical algorithms rely on the (approximately) monotonic learning process of iterative \textit{Policy Improvement} \citep{sutton2018reinforcement}.
Examples of algorithms that can be abstracted to this process are 
(Deep) Q-learners \citep{q_learning,dqn}, Actor-Critics \citep{ac_algs}, learners-in-the-model \citep{dreamerv3}, planning-based learners \citep{AlphaZero,muzero} and many others.
To drive policy improvement, RL algorithms use \textit{Policy Improvement Operators} which take the agent's current policy and produce a better policy \citep{viac,greedification_kls}.
Prominent examples of policy improvement operators are Policy Gradient \citep{sutton2000policy}, Monte Carlo Tree Search \citep{UCT} and the greedy operator used by Q-Learning.

A central design question in RL is which policy improvement operator best uses the information available to the agent under practical constraints. 
Approximate evaluation is particularly important because of both its prevalence and its adverse effects on learning \citep{baird1995residual,tsitsiklis1997analysis,thrun2014issues,
vanhasselt2010double,double_dqn,van2018deep}.
Popular algorithms mitigate the effects of evaluation error on policy improvement by constraining policy updates through different forms of regularization \citep{ppo,sac1,mpo,gumbel_mz,dreamerv3}.
These methods implicitly encode decisions about how uncertain evaluation should affect improvement. Making those decisions explicit and deriving operators that optimize the resulting objective remains an important avenue for research.

In this work we take steps towards optimal policy improvement under practical constraints such as approximate evaluation, in two central parts.
First, we characterize optimal improvement under state-restricted policy updates. We show that constrained optimal improvement is equivalent to solving an induced MDP.
These results connect optimal policy improvement to planning with explicit models such as world models and implicit models such as replay buffers.

To solve the MDP in the model, practical RL algorithms rely on \textit{greedification} for policy improvement \citep{sutton2018reinforcement,viac}.
This motivates the second part of our analysis: optimal greedification under approximate evaluation.
We model this problem explicitly as a problem of optimization under uncertainty formulated as Maximum a Posteriori optimization, and derive $\I_{\text{opt}}$, a novel greedification operator which is optimal for the derived optimization objective.
$\I_{\mathrm{opt}}$ generalizes an operator arising from Regularized Policy Optimization \citep{neu2017unified}, Control as Inference \citep{levine2018cai}, Mirror Descent \citep{mirror_descent_paper}, and Trust Region Policy Optimization \citep{trpo}, and used by popular algorithms such as GumbelAlpha/MuZero \citep{gumbel_mz} and TDMPC variants \citep{td_mpc2}. 

Theoretically, we investigate the soundness of $\I_{\text{opt}}$ for driving the RL
process by establishing convergence in a Dynamic Programming setup underlying the general learning process of RL \citep{viac}.
Empirically, $\I_{\text{opt}}$ and its practical gradient-based approximations improve performance across experiments encompassing RL and Dynamic Programming, discrete and continuous action spaces, model based and model free, classical and modern, online and offline algorithms.
These include GumbelAlphaZero \citep{gumbel_mz} in \textit{9x9 Go} and \textit{Chess}, Soft Actor-Critic \citep{sac1} in 10 classical DeepMind Control \citep{dm_control} continuous-action environments, ReBRAC \citep{tarasov2023revisiting} in 4 classical offline RL domains of D4RL \citep{d4rl} and Policy Iteration algorithms \citep{sutton2018reinforcement} across infinite horizon grid environments.
These results support the design of improved policy improvement operators as a promising basis for better performing RL algorithms, and in addition, the construction of explicit, uncertainty-aware decision objectives as a promising basis for designing stronger policy improvement operators.

\section{Preliminaries}
\label{sec:background}
In RL, the environment is represented by an
MDP
$ M = \langle \mathcal{S}, \mathcal{A}, \rho, R, P,
\gamma \rangle \in \M $, where $ \mathcal{S} $ is a set of states, $ \mathcal{A}$ a set of actions, $ \rho $ an initial state distribution.
For simplicity of presentation, we will focus on discrete and finite state and action spaces.
$ R: \mathcal{S} \times \mathcal{A} \to \mathbb{R} $ is a bounded possibly stochastic reward function, and $ P $ is a transition distribution such that $P(s' | s, a)$ specifies the probability of transitioning from state $s$ to state~$s'$ with action $a$.
A policy $\pi \in \Pi$ is a stochastic mapping from states $s \in \mathcal{S}$ to probability distributions on the action space $\mathcal{A}$. 
We write $a \sim \pi(s)$ and denote the probability of $a$ at $s$ under $\pi$ with $\pi(a|s)$.
The objective of RL is to find a policy which maximizes \textit{the expected discounted return} (or \textit{value}) $ V^\pi $: 
\begin{align}
    \label{eq:rl_objective}
    \mathcal J^\pi
    :=
    \E_{s \sim \rho} [V^\pi(s)]
    := \E\left[
        {\textstyle \sum_{t=1}^{H}} \gamma^{t-1} R(s_t, a_t)
    \right],
    s_1 \sim \rho, a_t \sim \pi(\cdot | s_t), s_{t+1} \sim P(\cdot \mid s_t, a_t).
\end{align}
The discount factor $\gamma \in (0,1)$ is introduced in the infinite-horizon setting ($H \to \infty$) to ensure that the value function $V^\pi$ is well-defined and bounded. 
In the finite-horizon case, one typically sets $\gamma = 1$. The state-action value function (Q-function) associated with a policy $\pi$ is defined by
$
Q^\pi(s,a) 
:= \E\Big[
    R(s,a) + \gamma V^\pi(s')
\Big],
$
where $s' \sim P(\cdot \mid s,a)$. 
The optimal state-value function, and state-action value function are defined, respectively, by
$
V^*(s) := \sup_{\pi \in \Pi} V^\pi(s),
Q^*(s,a) := \sup_{\pi \in \Pi} Q^\pi(s,a),
$
and an optimal policy $\pi^*$ satisfies $ \forall s\in \S: V^{\pi^*}(s) = V^*(s). $
The advantage function 
is defined as:
\begin{align}
    A^\pi(s,a) 
    = Q^\pi(s,a) - V^\pi(s) 
    = Q^\pi(s,a) -{\textstyle \sum_{b \in \A}}\pi(b|s)Q^\pi(s,b).
\end{align}
With access to the dynamics of the environment $M$, MDPs can be solved using Dynamic Programming (DP) methods, such as Value Iteration or Policy Iteration \citep{neuro_dynamic_programming}.
These methods can be cast as instances of \textit{Generalized Policy Iteration} (GPI, Algorithm \ref{alg:gpi}),
which 
iterate \textit{policy improvement} and \textit{policy evaluation} using different \textit{policy improvement operators} and underpin the learning process of many RL algorithms \citep{sutton2018reinforcement,viac}.
The policy improvement operator is generally defined as follows \citep{sutton2018reinforcement}:
\begin{definition}[Policy Improvement Operator]
\label{def:pio}
We call an operator $\mathcal{I} : \Pi 
\to \Pi$ a \emph{policy improvement operator} 
if, $\forall \pi \in \Pi$ as long as the policy is not already optimal, it satisfies:
\begin{align}
        \forall s \in \S: \,\,
        V^{\I(\pi)}(s) \geq V^{\pi}(s) 
        \quad 
        \text{and}
        \quad 
        \exists s \in \S: \,\, V^{\I(\pi)}(s) > V^{\pi}(s).
\end{align}
\end{definition}
A popular class of policy improvement operators are \emph{greedification operators}:
\begin{definition}[Greedification Operator]
    We call $\I : \Pi \times \mathcal{Q} \to \Pi$ a greedification operator if $\forall \pi \in \Pi, \, \forall Q \in \Q, \, \forall s \in \S$ it satisfies:
    \begin{align}
        {\textstyle \sum_{a \in \A}} \I(\pi, Q)(a|s) Q(s,a) 
        \,\geq\,        
        {\textstyle \sum_{a \in \A}} \pi(a|s) Q(s,a), 
         \label{eq:nonstrict_greedification}
    \end{align}
    and $\exists s \in \S $ 
    such that the inequality is strict,
    unless $\pi$ is already greedy with respect to $Q$.
    \label{def:go}
\end{definition}
$\Q$ denotes the set of all bounded functions $Q: \S\times\A \to \mathbb{R}$, to indicate that greedification is well defined with respect to exact evaluation $Q = Q^\pi$ as well as approximate evaluation $Q = \hat Q^\pi \approx Q^\pi $.
When $Q = Q^\pi$, \textit{greedification} (Inequality \ref{eq:nonstrict_greedification}) guarantees policy improvement \citep{sutton2018reinforcement}.
A popular example of a greedification operator is the greedy operator $\I_{\argmax}(\pi,Q)(s) = \argmax_{a \in \A} Q(s,a)$ used by Q-learning and Value / Policy Iteration.
We define $\argmax_{a \in \A}$ generally as any distribution which has support only over maximizing actions.
Another popular example is $\I_{reg}(\pi,Q)(a|s) \propto \pi(a|s)\exp(Q(s,a)/\eta) $ \citep{neu2017unified,levine2018cai,mirror_descent_derivations} which satisfies Definition \ref{def:go} when $\pi$ has support over the entire domain $\A$.

\section{A Path to Optimal Policy Improvement}
\label{sec:towards_opt_pi}

Our aim in this work is to chart a path towards RL algorithms driven by policy improvement operators which are optimal under practical constraints. 
We first show that under a natural definition and a common constraint, optimal improvement is equivalent to solving an induced MDP in the model. 

\subsection{Planning as a Foundation for Optimal Policy Improvement}
\label{sec:optimal_improvement_requires_planning}
We begin with a definition of optimal improvement:
\begin{definition}[Optimal Policy Improvement]
\label{def:optimal_pi}
We call $\I^*$ an \emph{optimal policy improvement operator} if it satisfies the following:
\begin{align}
\forall \pi \in \Pi, \forall s \in \S: \quad
    \I^*
    \in
    {\textstyle \arg\inf_{\I}} \left [V^*(s) - V^{\I(\pi)}(s)\right ]
    =
    {\textstyle \arg\sup_{\I}} \left [V^{\I(\pi)}(s) - V^\pi(s)\right].
\end{align}
\end{definition}
In fact, the constants $ V^*$ and $V^\pi$ do not influence the optimization problem, and are included only to emphasize that \textit{improvement} is equivalent to \textit{minimizing the distance from the optimum} as well as \textit{maximizing the increase over any policy $\pi$}.
Without imposing constraints on $\I$, this natural definition immediately makes the optimal improvement equivalent to simply solving the MDP: $ \I^*(\pi)(s) = \pi^*(s) $.
Such operators are not artificial, and can be interpreted as finding the optimal policy using exhaustive planning \citep{psrl,track_and_stop} or training to global convergence with a replay buffer that contains every transition in the environment, which itself can be viewed as planning with a pseudo-model \citep{sutton2018reinforcement,benefits_of_mbrl_for_gen}.

However, in practice improving the policy at all states at the same time is often impractical due to the size of $|\S|$ or lack of coverage of $\S$ in the data.
In this case, we can think of operators as constrained to improving the policy at a batch of states $S \subseteq \S $.
We find that \textit{even when limiting the operator to improving only in} $ S \subset \S $, optimal improvement is equivalent to solving a (semi-)MDP $M^{semi}_{S,\pi}$:
\begin{theorem}
\label{thm:solving_skip_mdp_is_opt}(Informal)
Optimal policy improvement constrained by $S \subset \S$, such that
$\forall s \notin S : \I(\pi)(s)=\pi(s)$, is equivalent to solving the
induced semi-MDP $M^{\mathrm{semi}}_{S,\pi}$.
\end{theorem}
\textbf{Intuition:}
The semi-MDP $M^{\mathrm{semi}}_{S,\pi}$ can be understood as equal to
$M$ on the decision states in $S$. Whenever a transition leaves $S$, the
resulting excursion is collapsed into a temporally extended transition:
the agent follows $\pi$ until it next returns to $S$, while preserving the
discounted return accumulated along the excursion and the corresponding
discounted continuation value. 
We provide formal definitions and proof in Appendix~\ref{proof:solving_skip_mdp_is_opt}.
As any discounted SMDP admits an equivalent MDP representation
\citep{sutton1999semi_MDPs}, we refer to solving this induced SMDP interchangeably as solving an induced MDP.

Constructing $M_{S,\pi}^{semi}$ requires knowledge of transitions outside of $S$, however.
Thus, it does not directly represent the setting where the agent only has access to a dataset composed of transitions in the environment which were gathered during training and are memorized in a replay buffer or encoded in a learned model.
More closely related to this setting, we can define the sub-MDP $M^{sub}_{S,\pi}$ on $S$ as the MDP which maintains the original dynamics for transitions within $S$ but terminates whenever a transition leaves $S$ with returns bootstrapped $R(s,a)+\gamma V^\pi(s')$. 
Thus, $M^{sub}_{S,\pi}$ optimizes within $S$, but does not model the trajectory after leaving $S$.
In this case, we have a similar result:
\begin{theorem}(Informal)
    \label{thm:solving_sub_mdp_is_lower_bound_opt}
    Optimal policy improvement constrained by $S \subset \S$ is equivalent to solving the sub-MDP $M^{sub}_{S,\pi},$ if leaving $S$ precludes returning to $S$.
    More generally, solving the sub-MDP $M^{sub}_{S,\pi}$ lower bounds the optimum and upper bounds a single step greedification followed by the policy $\pi$, i.e. $\max_a Q^\pi(s,a)$.
\end{theorem}
\textbf{Intuition:}
If leaving $S$ precludes returning to $S,$ the sub and semi MDPs coincide. 
If it is possible to leave and then return to $S,$ both planning as well as single step greedification can only account for the future outside of $S$ through the value function. 
We provide a formal description of $M^{sub}_{S,\pi}$ and a full proof in
Appendix~\ref{proof:solving_sub_mdp_is_lower_bound_opt}.

Such an operator can be interpreted as approximated by learning a model for transitions in $S$ and conducting search inside this model, such as MuZero \citep{muzero}, or training in this model, such as Dreamer \citep{dreamerv3}, or repeatedly training on transitions in $S\times\A$ while bootstrapping from $V^\pi(s')$ when the successor lies outside $S$. 

\textbf{We conclude:}
\emph{constrained optimal policy improvement is equivalent to solving an induced MDP}, directly connecting optimal improvement to planning with an explicit or implicit model.

\subsection{Greedification as a Driver of Planning-Based Policy Improvement}
\label{sec:optimal_improvement_to_optimal_greedification}
Theorems \ref{thm:solving_skip_mdp_is_opt} and \ref{thm:solving_sub_mdp_is_lower_bound_opt} show that optimal improvement is equivalent to, or bounded by, solving MDPs induced by the model, but do not show us \textit{how} to solve these MDPs.
In the RL literature, the popular choice for driving policy improvement within the model is greedification (Definition \ref{def:go}) \citep{sutton2018reinforcement,viac}.
Thus, a natural starting point for designing optimal policy improvement operators is the investigation of \textit{optimal greedification}.
We define \textit{optimal greedification} as satisfying the objective:
\begin{align}
    \argmax_{\pi'} \sum_{a \in \A} Q^\pi(s,a)\pi'(a|s) \,\, = \,\,
    \argmax_{\pi'} \sum_{a \in \A} A^\pi(s,a)\pi'(a|s),
    \label{eq:optimal_greedification}
\end{align}
i.e. maximizing with respect to the exact evaluation $A^\pi(s,a) = Q^\pi(s,a) - V^\pi(s) $ at a state $s$.
Under exact evaluation, Equation~\ref{eq:optimal_greedification} defines an unambiguous greedification objective, whose solutions are greedy policies, with $\I_{\argmax}$ naturally emerging as the optimal greedification operator (see Lemma \ref{lem:exact_optimal_greedification} in Appendix \ref{app:theory:greedy_is_opt}). 
When evaluation is approximate however, optimization should account for the evaluation uncertainty and optimal greedification must therefore be defined relative to specific decision formulations, which we address next.

\section{Optimal Greedification under Approximate Evaluation}
\label{sec:optimal_greedification}
In this section we formulate greedification under approximate evaluation as a probabilistic decision problem, make the required modeling choices and approximations to arrive at a tractable decision objective and derive a novel greedification operator which is optimal with respect to the derived decision problem.
We then proceed to establish that the operator is sound for driving the iterative learning process underlying practical RL algorithms by showing that it is sufficient for convergence to the optimal policy in a popular setup for such analysis \citep{tsitsiklis2002convergence,viac}.  

\subsection{Deriving the operator}
Approximate evaluation $\hat Q^\pi\approx Q^\pi$ induces uncertainty over which candidate policy $\pi'$ maximizes $\sum_{a\in\mathcal A}\pi'(a\mid s)A^\pi(s,a)$. Translating this uncertainty into a policy update requires specifying a decision problem. 
We develop an explicit formulation of such a decision problem based on \textit{generalized Bayesian inference} \citep{bissiri2016general}: the current policy supplies a prior over candidate policies, while an uncertainty-marginalized utility factor represents the evidence favoring each candidate. 
This gives the generalized posterior:
\begin{align}
    p_G\!\left(
        \pi'
        \mid
        \hat A^\pi(s,\cdot),\pi,s
    \right)
    \propto
    \underbrace{
        \mathcal L(
            \pi';
            \hat A^\pi(s,\cdot),\pi,s
        )
    }_{\text{Utility factor}}
    \underbrace{
        p(\pi'\mid\pi,s)
    }_{\text{Prior}}.
    \label{eq:generalized_posterior}
\end{align}
\textbf{Modeling the Prior.}
The current policy $\pi$ distills the behavior learned so far. Absent
additional information, we favor policies close to $\pi$, modeled using
a KL divergence with strength $\beta>0$:
\begin{align}
    p(\pi'\mid\pi,s)
    \propto
    \exp\left(
        -\beta_0
        \operatorname{KL}\!\left(
            \pi'(\cdot| s)
            \,\middle\|\,
            \pi(\cdot| s)
        \right)
    \right).
    \label{eq:policy_prior}
\end{align}
\textbf{Modeling the Utility Factor.}
We model the evaluation error as an unbiased heteroscedastic Gaussian:
\begin{align}
    \hat Q^\pi(s,a)
    =
    Q^\pi(s,a)+\epsilon(s,a),
    \qquad
    \epsilon(s,a)
    \sim
    \mathcal N\!\left(
        0,
        \sigma_{\hat Q^\pi}^2(s,a)
    \right),
    \label{eq:q_error_model}
\end{align}
independent across states and actions. 
The modeling choice of heteroscedastic unbiased Gaussian error is motivated by more than analytical tractability. 
\citet{hui2023doublegumbel} motivate Gaussian TD error
through the Central Limit Theorem and empirically find a moment-matched
heteroscedastic Gaussian to provide a suitable approximation. 
Classical MCTS \citep{UCT} was designed around unbiased sample means of
bounded Monte Carlo returns, which are asymptotically Gaussian. 
Despite the DNN-based evaluation used in modern MCTS \citep{AlphaZero,muzero}, MCTS remains highly successful. 
In Appendix~F.2, we similarly find that the residual errors in our MCTS experiments are reasonably modeled by a heteroscedastic Gaussian.
We discuss the modeling assumptions further in Section~4.2. 

The uncertain advantage
$
    \hat A^\pi(s,a)
    =
    \hat Q^\pi(s,a)
    -
    \sum_{b\in\mathcal A}
    \pi(b|s)\hat Q^\pi(s,b)
$
therefore satisfies the action-wise observation model:
\begin{align}
    \hat A^\pi(s,a)
    |
    A^\pi(s,a),\pi,s
    \sim
    \mathcal N \Big(
        A^\pi(s,a),
        \underbrace{
        (1-\pi(a|s))^2
        \sigma_{\hat Q^\pi}^2(s,a)
        +
        \sum_{b\neq a}
        \pi(b|s)^2
        \sigma_{\hat Q^\pi}^2(s,b)
        }_{
            =:\sigma_{\hat A^\pi}^2(s,a)
        }
    \Big).
    \label{eq:advantage_likelihood}
\end{align}
For analytical tractability, we model the action-wise advantage prior independently across actions as not favoring any action over another
$ A^\pi(s,a)\sim\mathcal N(0,\tau^2). $
Applying Bayes' rule to the advantage likelihood and prior using
Normal--Normal conjugacy gives the action-wise posterior:
\begin{align}
    p\!\left(
        A^\pi(s,a)
        \mid
        \hat A^\pi(s,a),\pi,s
    \right)
    =
    \mathcal N\bigg(
        \frac{\tau^2}{
            \tau^2+\sigma_{\hat A^\pi}^2(s,a)
        }
        \hat A^\pi(s,a),
        \frac{
            \tau^2\sigma_{\hat A^\pi}^2(s,a)
        }{
            \tau^2+\sigma_{\hat A^\pi}^2(s,a)
        }
    \bigg).
    \label{eq:actionwise_advantage_posterior}
\end{align}
Equations~\ref{eq:advantage_likelihood} and \ref{eq:actionwise_advantage_posterior} give action-wise posteriors conditioned on $\hat A^\pi(s,a)$. 
Although the $Q$-value errors are independent across actions, the advantage errors are correlated through their shared dependence on $\hat Q^\pi$. 
To combine the action-wise posteriors tractably, we adopt the approximate joint posterior across all actions:
\begin{align}
    \tilde p\!\left(
        A^\pi(s,\cdot)
        \mid
        \hat A^\pi(s,\cdot),\pi,s
    \right)
    &:=
    \prod_{a\in\mathcal A}
    p\!\left(
        A^\pi(s,a)
        \mid
        \hat A^\pi(s,a),\pi,s
    \right).
    \label{eq:factorized_advantage_posterior}
\end{align}
For a candidate policy $\pi'$, we define its utility relative to $\pi$ as
$    A^{\pi',\pi}(s):= \sum_{a\in\mathcal A}\pi'(a\mid s)A^\pi(s,a),
$
the maximization objective of greedification. Following the widely used exponential transformation of utility in decision models \citep{train2009discrete,bissiri2016general,levine2018cai}, we map this utility to an exponential score. 
Because the true advantage vector \(A^\pi(s,\cdot)\) is unobserved, the utility factor entering the generalized posterior in Equation~\ref{eq:generalized_posterior} is obtained by marginalizing this score over the approximate joint advantage posterior, where \(\lambda>0\) controls the scale of the exponentiated utility:
\begin{align}
    \mathcal L\!\left(
        \pi';
        \hat A^\pi(s,\cdot),\pi,s
    \right)
    &:=
    \mathbb E_{
        A^\pi(s,\cdot)
        \sim
        \tilde p(
            \cdot
            \mid
            \hat A^\pi(s,\cdot),\pi,s
        )
    }
    \left[
        \exp\left(
            \lambda A^{\pi',\pi}(s)
        \right)
    \right].
    \label{eq:marginalized_utility}
\end{align}
Evaluating this expectation under the factorized Gaussian posterior introduces quadratic terms $\lambda^2(\pi'(a\mid s))^2$ through the Gaussian moment-generating function, precluding a simple closed-form solution to the resulting MAP. 
To retain tractability, we adopt a certainty-equivalent approximation \citep{vandewater1981certainty}, replacing each uncertain advantage by its posterior mean.
We denote the certainty-equivalent approximation to the
uncertainty-marginalized utility factor by:
\begin{align}
    \mathcal L_{\mathrm{CE}}\!\left(
        \pi';
        \hat A^\pi(s,\cdot),\pi,s
    \right)
    &:=
    \exp\left(
        \lambda
        \sum_{a\in\mathcal A}
        \pi'(a|s)
        \frac{\tau^2}{
            \tau^2+\sigma_{\hat A^\pi}^2(s,a)
        }
        \hat A^\pi(s,a)
    \right)
    \approx
    \mathcal L\!\left(
        \pi';
        \hat A^\pi(s,\cdot),\pi,s
    \right).
    \label{eq:certainty_equivalent_utility}
\end{align}
The corresponding tractable generalized posterior is
\begin{align}
    p_{\mathrm{CE}}\!\left(
        \pi'
        \mid
        \hat A^\pi(s,\cdot),\pi,s
    \right)
    &\propto
    \mathcal L_{\mathrm{CE}}\!\left(
        \pi';
        \hat A^\pi(s,\cdot),\pi,s
    \right)
    p(\pi'\mid\pi,s).
    \label{eq:certainty_equivalent_posterior}
\end{align}
Only the ratio $\beta = \beta_0/\lambda$ affects the MAP solution, which we use in the rest of the derivation for simplicity of notation. 
The log-posterior then induces the following tractable decision objective:
\begin{align}
    \mathcal J_{\mathrm{CE}}\!\left(
        \pi';
        \hat A^\pi(s,\cdot),\pi,s
    \right)
    \!:=\!\!
    \sum_{a\in\mathcal A}
    \pi'(a\mid s)
    \frac{\tau^2}{
        \tau^2+\sigma_{\hat A^\pi}^2(s,a)
    }
    \hat A^\pi(s,a)
    -
    \beta
    \KL\!\left(
        \pi'(\cdot| s)
        \,\middle\|\,
        \pi(\cdot| s)
    \right).
    \label{eq:ce_decision_objective}
\end{align}
We define $\I_{\mathrm{opt}}$ as the operator which solves this decision problem:
\begin{align}
    \I_{\mathrm{opt}}(\pi,\hat Q^\pi)(\cdot|s)
    &:=
    \argmax_{\pi'\in\Delta(\mathcal A)}
    \mathcal J_{\mathrm{CE}}\!\left(
        \pi';
        \hat A^\pi(s,\cdot),\pi,s
    \right).
    \label{eq:iopt_definition}
\end{align}
Assuming that $\pi$ has full support over $\mathcal A$, the objective
in Equation~\ref{eq:ce_decision_objective} is strictly concave in
$\pi'(\cdot\mid s)$ and therefore has a unique global maximizer:
\begin{align}
    \I_{\mathrm{opt}}(\pi,\hat Q^\pi)(a|s)
    &\propto
    \pi(a|s)
    \exp\left(
        \frac{\tau^2}{
            \bigl(
                \tau^2+\sigma_{\hat A^\pi}^2(s,a)
            \bigr)\beta
        }
        \hat A^\pi(s,a)
    \right)
    \\
    &=
    \pi(a|s)
    \exp\left(
        \frac{\hat A^\pi(s,a)}{\eta(s,a)}
    \right),
    \,\,\ \text{where} \,\,\
    \eta(s,a)
    :=
    \frac{
        \bigl(
            \tau^2+\sigma_{\hat A^\pi}^2(s,a)
        \bigr)\beta
    }{\tau^2}.
    \label{eq:iopt}
\end{align}
\(\mathcal I_{\mathrm{opt}}\) is the unique MAP optimizer of the certainty-equivalent generalized posterior (Equation~\ref{eq:certainty_equivalent_posterior}), which itself approximates the original posterior formulated for the decision objective (Equation \ref{eq:generalized_posterior}).
Equivalently, \(\mathcal I_{\mathrm{opt}}\) is the optimal operator with respect to the derived objective (Equation~\ref{eq:ce_decision_objective}). 
In the no-uncertainty setting, where $ \forall a \in \A: \sigma_{\hat A^\pi}^2(s,a)=0$ and $\beta \to 0,$ $\I_{\mathrm{opt}}$ reduces to
$\I_{\arg\max}$.
For detailed derivations see Appendix \ref{app:derivations}.

\subsection{Theoretical Analysis}
We proceed to establish that $\I_{\text{opt}}$ is sound for driving the iterative learning process, in a setup used for analysis of the underlying learning behavior of RL algorithms \citep{tsitsiklis2002convergence,viac} (GPI, Algorithm~\ref{alg:gpi} in Appendix \ref{app:psudocode}).
We begin by showing that $\I_{\text{opt}}$ satisfies \textit{greedification}:
\begin{theorem}
    \label{thm:greedification}
    For any $\beta> 0, \tau^2 > 0, \sigma^2_{\hat Q^\pi}\geq 0 $ and policy $\pi$ with full support over $\A$, $\I_{\text{opt}}(\pi, \hat Q^\pi)$ is a greedification operator (Definition \ref{def:go}).
\end{theorem}
\textbf{Intuition:} 
$\I_{\text{opt}}$ exclusively \textit{increases} the unnormalized log-probability of \textit{advantage-positive} and \textit{decreases} the unnormalized log-probability of \textit{advantage-negative} actions, which guarantees greedification \citep{gumbel_mz}.
See Appendix \ref{proof:greedification} for full proof.

Using the greedification property, we show that $\I_{\text{opt}}$ is sufficient for convergence:
\begin{theorem}[Convergence]
\label{thm:convergence}
Let $\beta>0$ and $\tau^2>0$, $\pi_0$ have full support, the evaluation errors be fresh and independent conditional on the history $\mathcal F_i$:
$
\epsilon_i(s,a)\mid\mathcal F_i
\sim\N(0,\sigma_{Q_i}^2(s,a)),
$
where the corresponding variances are supplied to $\I_{\mathrm{opt}}$
and satisfy $\sigma_{Q_i}^2(s,a)\to0$.
Then, in the finite state, action and horizon GPI process (Algorithm~\ref{alg:gpi} with its $\delta \to 0$):
\begin{align}
    V^{\pi_i}\xrightarrow[i\to\infty]{\mathrm{a.s.}} V^*, \qquad
    \sum_{a\notin\argmax_b Q^*(s,b)} \pi_i(a|s)\xrightarrow[i\to\infty]{\mathrm{a.s.}}0,
    \quad \forall s.
\end{align}
\end{theorem}
\textbf{Intuition:}
$\I_{opt}$ satisfies a \textit{Limit-Sufficient Greedification} condition, established by \citet{viac} to induce convergence, almost surely under vanishing evaluation noise.
See Appendix~\ref{proof:convergence} for proof.

\textbf{Remark.}
Optimality under approximate evaluation is necessarily relative to a decision formulation: different objectives, likelihoods, priors, utilities, or approximations can yield different operators. 
Within this family of possible formulations, we choose modeling assumptions and approximations that yield a practical operator which accounts for evaluation uncertainty, is inexpensive to compute, and admits theoretical analysis (Theorems~\ref{thm:greedification} and~\ref{thm:convergence}). 
Our contribution is to instantiate an explicit, informed approach to operator design and solve the resulting tractable objective exactly. 
Possible interactions between a policy improvement operator which accounts for evaluation uncertainty and exploration are discussed in Appendix \ref{app:I_opt_and_exploration}.

\section{Connections to Prior Work}
\label{sec:related_work}
$\I_{\text{opt}}$ generalizes the widely used shift-invariant $\I_{reg}$ operator parameterized by $\eta$:
\begin{align*}
     \I_{reg}(\pi,Q)(a|s) 
     \propto 
     \pi(a|s) \exp \left(Q(s,a)/\eta \right) 
     \propto
     \pi(a|s) \exp\left(A(s,a)/\eta\right).
\end{align*}
$\I_{reg}$ itself emerges independently from several seminal directions: \emph{Regularized Policy Optimization} \citep{neu2017unified}; \emph{Control as Inference} and \emph{``Optimal Control"} \citep{levine2018cai,optimal_control}, which interpret $\pi'(s) = \I_{reg}(\pi,Q^{\pi'}_{soft})(s)$ as a ``soft-optimal" policy and motivate
SAC \citep{sac1},
MPO \citep{mpo} and TDMPC \citep{td_mpc,td_mpc2}; The \emph{Natural Policy Gradient} \citep{npg}, \emph{Trust Region Policy Optimization} \citep{trpo} and \emph{Mirror Descent} \citep{mirror_descent_paper} whose improvement steps have all been shown to connect to $\I_{reg}$ \citep{neu2017unified,mirror_descent_derivations}. 
For additional detail see Appendix \ref{sec:more_related_work}.

In view of $\I_{\mathrm{opt}}$, the regularization term used by Regularized Policy Optimization induces a prior over candidate policies. 
From the perspective of Control as Inference, $\I_{\mathrm{reg}}$ is most directly interpreted as a form of optimal \textit{improvement}, rather than as approximating a globally optimal \textit{policy}. 
Moreover, \(1/\eta(s,a)\) can be interpreted as an action-dependent step size which accounts for the uncertainty in the improvement signal.
Conversely, when we do not have any information about the evaluation error we can model the uncertainty as uniformly bounded by a constant, use $\sigma_{\hat A}^2(s,a) = \sigma^2$ conservatively, reducing $\eta(s,a) $ to $\eta$ and $\I_{\text{opt}}$ to $\I_{reg},$ which thus can be viewed as optimal with respect to the same decision objective as $\I_{\text{opt}},$ when we are uninformed about error.

Gumbel-AlphaZero/MuZero (GA/MZ) \citep{gumbel_mz} use $\I_{reg}$ within MCTS, where the coefficient $1/\eta(s)$ increases with $N(s,a),$ the number of evaluations at each node in the search tree: $1/\eta(s) = (c_{\text{visit}} + \max_a N(s,a))c_{\text{scale}}$.
This heuristic mechanism is inspired by the original MCTS assumptions which induce $N(s,a) \propto 1/\sigma^2_{\hat Q^\pi}(s,a)$ and increases as evaluation error reduces \citep{UCT}. 
Rather than using $N(s,a)$ directly which does not guarantee policy improvement, GA/MZ use $ \max_a N(s,a)\propto \max_a 1/\sigma^2_{\hat Q^\pi}(s,a) $: a per-\emph{node} surrogate for a quantity that is inherently per-\emph{action}.
$\I_{\text{opt}}$ can circumvent this heuristic and use $ N(s,a) $ per action directly.

Several previous works incorporate uncertainty into policy optimization for
specific algorithms or sources of uncertainty.
POMBU \citep{zhou2020deep} uses $Q$-value uncertainty to make PPO updates
more conservative.
\citet{queeney2021uncertainty} (UA-TRPO) make TRPO robust in the estimated policy gradient to uncertainty sourced from finite sampling.
UWAC \citep{wu2021uncertainty}, specifically addressing offline RL, uses
epistemic uncertainty to down-weigh the contribution of out-of-distribution predictions in actor-critic training.
\citet{hui2023doublegumbel} model $Q$-function evaluation error probabilistically
and derive a modified Q-learning objective.
In contrast, we study policy improvement at the operator level rather
than within a particular algorithmic framework. 
We formulate greedification under uncertain evaluation as an explicit optimization problem and derive an operator that is optimal for the resulting uncertainty-aware objective under the stated probabilistic modeling and approximation choices. 
The operator admits convergence guarantees in the analyzed GPI setting and can be incorporated modularly into algorithms with different policy-improvement mechanisms (see Section~\ref{sec:experiments}).

\section{Experiments}
\label{sec:experiments}
$\I_{\mathrm{opt}}$ exactly solves the decision objective derived in Equation~\ref{eq:ce_decision_objective}. 
Our experiments investigate whether optimizing this local decision objective indeed facilitates improvement in the broader iterative learning process. 
We ask two questions. 
First, \textit{does $\I_{\mathrm{opt}}$ improve sample efficiency in a controlled Policy Iteration setting that matches its evaluation assumptions}? Second, \textit{do $\I_{\mathrm{opt}}$ and its gradient-based practical approximations improve established RL algorithms when uncertainty must be estimated during training}? We study these questions across Dynamic Programming, model-based and model-free RL, discrete and continuous actions, and online and offline learning.

$\I_{\mathrm{opt}}$'s hyperparameters were tuned per baseline algorithm it was incorporated into and in a similar manner in which each baseline algorithm was itself tuned. 
For additional details see Appendix \ref{app:hp_tuning}.

\begin{figure*}[ht]
    \centering
    \includegraphics[width=1\linewidth]{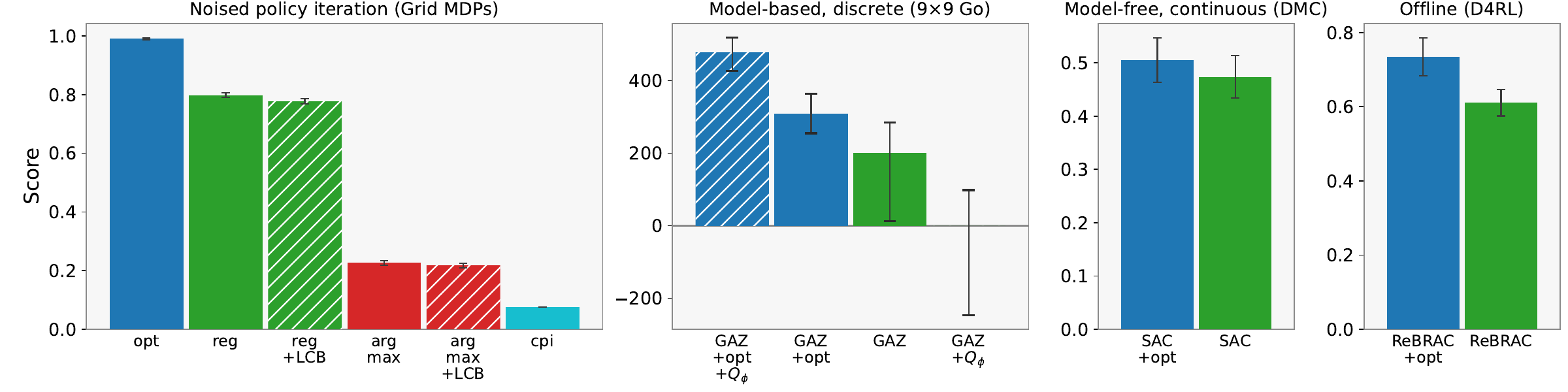}
    \vspace{-0.7cm}
    \caption{
    \textbf{Noised policy iteration:} normalised $V^{\pi}(s_0)$ at iteration $40$ over $20$ grid environments, 95\% Gaussian CI. 
    \textbf{Model-based, discrete:} Bayes Elo at $34$M training frames, mean, min. and max. over $5$ training seeds in 9x9 Go.
    \textbf{Model-free, continuous:} 
    mean normalised returns at $1$M steps on $10$ DeepMind Control Suite tasks, mean and 95\% Gaussian CI. 
    \textbf{Offline:} normalised return at $1$M steps on $4$ D4RL medium-replay tasks, mean and 95\% Gaussian CI.
    }
    \label{fig:flagship}
\end{figure*}
\textbf{Policy Iteration (Grid MDPs).}
We compare Policy Iteration algorithms with different policy improvement operators $\I$ (Algorithm \ref{alg:gpi} with $k \to \infty$) with evaluation error which coincides with the assumptions under which $\I_{\text{opt}}$ is designed.
At each iteration $i$ the exact evaluation $Q^{\pi_i}$ is
computed analytically and corrupted by heteroscedastic noise $ \hat Q^{\pi_i}(s,a) \sim \N(Q^{\pi_i}(s,a), \sigma^2_{Q_i}(s,a)). $
The variances $\sigma^2_{Q_i}(s,a)$ are sampled each iteration from a decaying schedule.
We compare $\I_{\text{opt}}$ to the following standard operators: $\I_{\arg\max}$ (standard policy iteration), $\I_{\text{reg}}$ and $\I_{\text{cpi}}$ \citep{safe_pi}, a variant of Conservative Policy Iteration \citep{approx_optimal_approx_rl}, a method designed for Policy Iteration in the presence of approximate evaluation.
We include pessimistic variants of $\I_{\arg\max}$ and $\I_{\text{reg}}$ that use the lower-confidence bound (LCB) $\hat Q^{\pi_i}(s,a) - \scriptstyle\lambda_{\text{LCB}}\sqrt{\sigma^2_{Q_i}(s,a)}$ to test whether $\I_{\text{opt}}$'s uncertainty-awareness merely reproduces standard LCB-based pessimism. 

$\I_{\text{opt}}$ outperforms every operator tested (Figure \ref{fig:flagship},
left) in the presence of the evaluation error of the type for which it was designed to account.
The effect of $\I_{\text{opt}}$ is clearly different from that of LCB-based pessimism, which does not contribute to performance in this setting.
See Appendix \ref{app:details:dp} for implementation and experimental details and Appendix \ref{app:additional_results_dp} for additional results including the learning curves, an investigation of the effect of the magnitude of the noise, sensitivity to the hyperparameters $\tau, \beta$ and an experiment with biased evaluation rather than unbiased.

\textbf{Model-based discrete action (GAZ on board games).}
In Section \ref{sec:towards_opt_pi}, we showed how greedification connects to search for policy improvement.
With this experiment, we circle back to this connection and evaluate $\I_{\text{opt}}$'s contribution to search-based policy improvement with the search operator GumbelMCTS used by the modern AlphaZero variant GumbelAlphaZero (GAZ), which remains a popular and well-performing RL algorithm \citep{gumbel_mz,EffZeroV2,pmcts}. 
As discussed in Section \ref{sec:related_work}, it is natural to replace the heuristic $\eta(s)$ GAZ uses in $\I_{\text{reg}}$ with the theoretically motivated $\I_{\text{opt}}$.
This enables the algorithm to account for the variance using $\sigma^2_{Q_i}(s_t,a) \propto 1/N(s,a)$ at each node \textit{and action} $(s_t,a)$ independently, as was enabled by the PUCT search policy \citep{puct,AlphaGo} which originally inspired using $\I_{\text{reg}}$ in MCTS \citep{grill2020monte}.
We include a comparison to VanillaAZ with PUCT for reference in Figure \ref{fig:mcts_M_grid} in Appendix \ref{app:additional_mcts_results}.

$\I_{\text{reg}}$ requires an estimate of $Q^\pi(s,a)$ for each state and action, which traditional MCTS architectures do not provide.
GumbelMCTS instead uses the \textit{completed Q-values} mechanism, which substitutes a prediction of $V^\pi(s)$ for $Q^\pi(s,a)$ for unsearched actions, composed of the predictions of all searched actions weighted under $\pi$.
This substitution removes action-specific information, since every unsearched action receives the same value prediction.
When action-wise prediction errors contain independent or weakly correlated components, this averaging reduces their variance.
Moreover, actions with greater policy support are generally searched or played, and consequently trained, more frequently; weighting by $\pi$ therefore tends to emphasize better-supported estimates while assigning less weight to weakly supported, higher-error actions.
Consistent with this reasoning, using the empirical proxy described in Appendix~\ref{app:additional_results_qv}, we observe lower estimated error variance for the $V^\pi$-based predictions in our experiments.
We therefore hypothesize that $\I_{\text{opt}}$, which is designed to exploit action-specific $Q_\phi$ predictions while accounting for their estimated uncertainty, will benefit more than baseline GAZ from incorporating a $Q_\phi$ critic.
Such a critic is inexpensive to incorporate as a small additional head on the shared-body DNN of GAZ.

Figure \ref{fig:flagship} (second from the left) reports Bayes Elo \citep{bayes_elo} on $9\times9$ Go after $34$M frames of training, averaged over five training seeds per agent.
$\I_{\text{opt}}$ benefits substantially from a $Q_\phi$ critic, while the baseline's performance \textit{degrades} with a $Q_\phi$ critic. 
$\I_{\text{opt}}$ mean Elo is 107.8 higher than the baseline's, even without the addition of the $Q_\phi$ critic.
With $Q_\phi$ the difference increases to 277 Elo.
See Appendix \ref{app:details:mcts} for additional experimental and implementation details and Appendix \ref{app:additional_mcts_results} for additional results, including per-seed results, an investigation of the error behavior in MCTS and a fixed-network comparison along compute budgets $M$ which includes the additional board game chess and the additional search policy PUCT.
In the fixed-network comparison, GumbelMCTS with $\I_{\mathrm{opt}}$ and the $Q_\phi$ critic achieves the highest playing strength among all tested GumbelMCTS and PUCT variants at every evaluated search budget in both $9\times9$ Go and chess
(Figure~\ref{fig:mcts_M_grid}).

\textbf{Model free classical continuous control (SAC on DMC).}
We incorporate $\I_{\text{opt}}$ into the classical continuous-control RL algorithm SAC \citep{sac1}.
In practical implementations, SAC approximates a projection of the target soft distribution $\I_{reg}(\mu,Q_{soft})(a|s) $ (where $\mu$ is the uniform policy and $Q_{soft}$ is the soft Q-value function) to the parameterized policy using empirical gradient descent.
We incorporate $\I_{opt}$ into this traditional implementation of SAC similarly, with a gradient based approximation of $\I_{opt}$ (see Appendix \ref{app:details:sac} for details). 
To estimate $\sigma^2_{Q} $ we use the rudimentary and cheap uncertainty estimation already provided by SAC's two critic value ensemble.
Results are aggregated across 10 classical continuous control environments from the DeepMind Control suite \citep[DMC,][]{dm_control}.
Figure \ref{fig:flagship} (second from right) reports normalised returns after $1$M steps. 
Even in the presence of the rudimentary uncertainty estimation provided by the value ensemble, which was not designed for explicit uncertainty estimation, the gradient-based approximation of $\I_{\text{opt}}$ is able to provide performance improvement over baseline SAC.
See Appendix \ref{app:details:sac} for implementation and experimental details and Appendix \ref{app:additional_results_SAC} for per-environment results.

\textbf{Offline RL (ReBRAC on D4RL).} We evaluate $\I_{\text{opt}}$'s contribution to performance in offline RL, incorporating it into the recent ReBRAC algorithm \citep{tarasov2023revisiting} similarly using the two critic value ensemble (see Appendix \ref{app:details:offline} for details).
We evaluate on the medium-replay datasets of the D4RL Gym-MuJoCo suite \citep{d4rl} using four locomotion environments: \textit{halfcheetah}, \textit{hopper}, \textit{walker2d} and \textit{ant} with the \textit{medium-replay} datasets. 
Scores are D4RL-normalized, where 0 corresponds to a random policy and 100 to an expert (SAC) policy. 
Each agent trains for $1$M gradient steps.
Figure \ref{fig:flagship} (right) reports the return of the final policy at the end of training. 
The gradient-based approximation of $\I_{\text{opt}}$ improves on the baseline in this domain as well. 
See Appendix \ref{app:details:offline} for implementation and experimental details and Appendix \ref{app:additional_results:offline} for per-environment results.

\section{Conclusions and Future Work}
\label{sec:conclusions}
In this work, we took steps towards practical RL algorithms driven by optimal policy improvement operators.
We characterized optimal policy improvement, showing that optimal improvement is itself equivalent to solving MDPs and fundamentally connecting optimal policy improvement to planning with an explicit and implicit model.
To solve the MDP in the explicit or implicit model, practical methods rely on \textit{greedification}, which lead us to focus on optimal greedification under practical constraints.
We formulated greedification under approximate evaluation as probabilistic decision-making under uncertainty and derived a tractable decision objective. 
Solving for this objective, we derived $\I_{\mathrm{opt}},$ a novel greedification operator which is optimal with respect to the probabilistic decision objective.
We established convergence for algorithms driven by $\I_{\mathrm{opt}}$ in a Generalized Policy Iteration setting and shown that $\I_{\mathrm{opt}}$ generalizes a popular greedification operator in literature, $\I_{\mathrm{reg}},$ extending it with explicit accounting for evaluation uncertainty.
Empirically, $\I_{\mathrm{opt}}$ and its practical gradient-based approximations improved aggregate performance over the respective GumbelAlphaZero, SAC, and ReBRAC baselines in board games, online continuous control, and offline continuous control, respectively. 
In a noised Policy Iteration experiment which coincides with the assumptions under which $\I_{\mathrm{opt}}$ was designed, the novel operator outperformed all other tested operators. 
$\I_{\mathrm{opt}}$ presents a new, attractive choice for greedification, unified across different algorithms, and lays a path towards optimal policy improvement through the design of operators which are optimal with respect to practical RL constraints.


\subsection*{AI use statement}
In this work, we used generative AI tools to assist with writing and
checking proofs, developing and implementing experimental code,
analyzing and plotting results, and supporting literature research. 
We did not use generative AI for any other task requiring disclosure. 
We reviewed all AI-assisted content and verified and tested all
AI-generated code. 
We take full responsibility for the final manuscript, including all text, claims, code, and other artifacts produced with AI
assistance.

\label{endofmaintext}
\bibliography{bib}
\bibliographystyle{iclr2027_conference}


\appendix
\etocdepthtag.toc{appendix}

\newpage

{
  \etocsettagdepth{main}{none}
  \etocsettagdepth{appendix}{subsection}
  \etocsettocstyle{\section*{Appendix Contents}}{}
  \tableofcontents
}

\newpage

\section{Proofs}
\label{app:proofs}

\subsection{Optimal policy improvement as solving an induced MDP}
\label{proof:solving_skip_mdp_is_opt}
Let $S \subseteq \S$ denote the subset of states at which the policy is
permitted to change, and let $\pi$ denote the current policy.
We construct an induced semi-Markov decision process (SMDP)
$M^{\mathrm{semi}}_{S,\pi}$ whose decision states are exactly the states in
$S$.

Starting from a decision state $s_0=s\in S$, the agent first selects an action
$a_0=a$. The environment then evolves according to the original transition
dynamics $P$. Whenever the resulting trajectory leaves $S$, actions are
subsequently sampled from the fixed policy $\pi$ until the process first
returns to $S$. Define this first-return time as
\begin{align}
    \tau
    :=
    \inf\{t\geq 1 : s_t \in S\}.
\end{align}
If the process transitions directly from one state in $S$ to another, then
$\tau=1$.

Each state-action pair $(s,a)$ in the induced SMDP therefore corresponds to a
temporally extended transition in the original MDP: take action $a$ at $s$,
and, whenever the process leaves $S$, follow $\pi$ until the next state in
$S$ is reached.

The reward of this temporally extended transition is the expected discounted
return accumulated before the next decision state:
\begin{align}
    R_{M^{\mathrm{semi}}_{S,\pi}}(s,a)
    :=
    \E\left[
        \sum_{t=0}^{\tau-1}\gamma^t R(s_t,a_t)
        \,\middle|\,
        s_0=s,\;a_0=a,\;\pi
    \right].
    \label{eq:semi_mdp_reward}
\end{align}
Here and below, if the process never returns to $S$, we take $\tau=\infty$,
so that the trajectory contributes its complete discounted return to
Equation~\ref{eq:semi_mdp_reward} and no continuation value.

The transition dynamics of the induced SMDP are described jointly by the
next decision state and the duration of the temporally extended transition:
\begin{align}
    P_{M^{\mathrm{semi}}_{S,\pi}}(s',k|s,a)
    :=
    \Pr\left(
        s_\tau=s',\;
        \tau=k
        \,\middle|\,
        s_0=s,\;a_0=a,\;\pi
    \right),
    \qquad
    s,s'\in S,\; k\geq 1.
    \label{eq:semi_mdp_transition}
\end{align}
Consequently, the Bellman optimality equation of
$M^{\mathrm{semi}}_{S,\pi}$ is
\begin{align}
    V^{*,\mathrm{semi}}(s)
    =
    \max_{a\in\A}
    \left[
        R_{M^{\mathrm{semi}}_{S,\pi}}(s,a)
        +
        \sum_{s'\in S}
        \sum_{k=1}^{\infty}
        P_{M^{\mathrm{semi}}_{S,\pi}}(s',k|s,a)
        \gamma^k
        V^{*,\mathrm{semi}}(s')
    \right].
    \label{eq:semi_mdp_bellman}
\end{align}
The factor $\gamma^k$ accounts for the possibly variable duration of the
temporally extended transition.

We can now re-describe
Theorem~\ref{thm:solving_skip_mdp_is_opt} more formally.

\begin{theorem}
    \label{thm:solving_skip_mdp_is_opt_formal}
    Let
    \begin{align}
        \Pi_S(\pi)
        :=
        \left\{
            \pi'\in\Pi :
            \pi'(\cdot|s)=\pi(\cdot|s),
            \ \forall s\notin S
        \right\}
    \end{align}
    denote the set of policies that may differ from $\pi$ only on
    $S\subseteq\S$.
    Let $\pi^{*,\mathrm{semi}}_S$ be an optimal policy of
    $M^{\mathrm{semi}}_{S,\pi}$, and extend it to the original MDP by
    \begin{align}
        \bar{\pi}(a|s)
        :=
        \begin{cases}
            \pi^{*,\mathrm{semi}}_S(a|s), & s\in S,\\
            \pi(a|s), & s\notin S.
        \end{cases}
    \end{align}
    Then $\bar{\pi}$ is an optimal policy among all policies in
    $\Pi_S(\pi)$. In particular,
    \begin{align}
        \forall s\in\S:
        \qquad
        V^{\bar{\pi}}(s)
        =
        \sup_{\pi'\in\Pi_S(\pi)}
        V^{\pi'}(s).
        \label{eq:semi_mdp_constrained_optimum}
    \end{align}
    Equivalently, since $V^*(s)\geq V^{\pi'}(s)$ for every policy $\pi'$,
    \begin{align}
        \forall s\in\S:
        \qquad
        \bar{\pi}
        \in
        \arg\inf_{\pi'\in\Pi_S(\pi)}
        \left[
            V^*(s)-V^{\pi'}(s)
        \right].
    \end{align}
    Thus, solving $M^{\mathrm{semi}}_{S,\pi}$ is equivalent to optimal
    policy improvement under the constraint that the policy may be modified
    only on $S$.
\end{theorem}

\textit{Proof.}
Consider any policy $\pi'\in\Pi_S(\pi)$. Whenever the process is in a state
$s\in S$, $\pi'$ selects an action according to $\pi'(\cdot|s)$. After this
action, if the trajectory leaves $S$, the policies $\pi'$ and $\pi$ coincide
by definition of $\Pi_S(\pi)$. Hence, until the trajectory next reaches a
state in $S$, its action distribution is exactly the one used in the
construction of $M^{\mathrm{semi}}_{S,\pi}$.

Therefore, every policy $\pi'\in\Pi_S(\pi)$ induces a policy on
$M^{\mathrm{semi}}_{S,\pi}$ given by its restriction to $S$.
Conversely, every policy on $M^{\mathrm{semi}}_{S,\pi}$ uniquely induces a
policy in $\Pi_S(\pi)$ by using that policy on $S$ and $\pi$ on
$\S\setminus S$.

By Equations~\ref{eq:semi_mdp_reward} and
\ref{eq:semi_mdp_transition}, each temporally extended transition in
$M^{\mathrm{semi}}_{S,\pi}$ preserves both the discounted reward accumulated
during the corresponding trajectory segment in the original MDP and the
discounted continuation value upon returning to $S$. Thus, under the
policy correspondence above,
\begin{align}
    V^{\pi'}_M(s)
    =
    V^{\pi'|_S}_{M^{\mathrm{semi}}_{S,\pi}}(s),
    \qquad
    \forall s\in S.
    \label{eq:semi_mdp_value_equivalence}
\end{align}

It follows that an optimal policy of $M^{\mathrm{semi}}_{S,\pi}$ maximizes
the value in the original MDP among all policies that differ from $\pi$ only
on $S$, for every $s\in S$.

Finally, consider any state $s\notin S$. Every policy in $\Pi_S(\pi)$ follows
the same policy $\pi$ until either the process terminates or first reaches
$S$. Therefore, the distribution of rewards and of the first state reached
in $S$ before that point is independent of the choice of policy within
$S$. Since $\bar{\pi}$ maximizes the continuation value at every state in
$S$, it also maximizes the value from every state $s\notin S$.
Hence,
\begin{align}
    V^{\bar{\pi}}(s)
    =
    \sup_{\pi'\in\Pi_S(\pi)}
    V^{\pi'}(s),
    \qquad
    \forall s\in\S.
\end{align}
Therefore, solving $M^{\mathrm{semi}}_{S,\pi}$ is equivalent to solving the
original MDP under the constraint that the policy equals $\pi$ on
$\S\setminus S$, proving the result.
\qed

\subsection{Solving sub-MDPs is a lower bound on optimal policy improvement}
\label{proof:solving_sub_mdp_is_lower_bound_opt}

Let $S \subseteq \S$ denote the subset of states at which the policy is
permitted to change, and let $\pi$ denote the current policy.
We define the sub-MDP $M^{sub}_{S,\pi}$ by preserving the original dynamics
within $S$ and terminating trajectories whenever they leave $S$.

Specifically, for a transition $(s,a,s')$ with $s\in S$ and $s'\in S$,
the original transition probability and reward are preserved.
For a transition with $s\in S$ and $s'\notin S$, the transition is made
terminal and bootstrapped using the value of the current policy at the
realized successor state $s'$. Thus, the value associated with such an
exit transition is
\begin{align}
    R(s,a) + \gamma V^\pi(s').
\end{align}
Importantly, this boundary value depends on the realized successor $s'$,
rather than on the unconditional state-action value $Q^\pi(s,a)$.

Equivalently, the Bellman optimality operator of $M^{sub}_{S,\pi}$ is,
for every $s\in S$,
\begin{align}
    \left(\mathcal{T}^{sub}_{S,\pi}V\right)(s)
    :=
    \max_{a\in\A}
    \Bigg[
        R(s,a)
        +
        \gamma
        \sum_{s'\in S}
        P(s'|s,a)V(s')
        +
        \gamma
        \sum_{s'\notin S}
        P(s'|s,a)V^\pi(s')
    \Bigg].
    \label{eq:sub_mdp_bellman_operator}
\end{align}
Let $V^{*,sub}_{S,\pi}$ denote the unique fixed point of
$\mathcal{T}^{sub}_{S,\pi}$.

The difference between $M^{sub}_{S,\pi}$ and the induced semi-MDP
$M^{semi}_{S,\pi}$ of the previous section is therefore that
$M^{semi}_{S,\pi}$ remains connected through the behavior of $\pi$
outside $S$, whereas $M^{sub}_{S,\pi}$ terminates immediately upon
leaving $S$ and bootstraps from $V^\pi$ at the boundary.

Consequently, constructing $M^{sub}_{S,\pi}$ does not require a transition
model beyond transitions originating from state-action pairs in
$S\times\A$. It does, however, require the boundary values $V^\pi(s')$
for successor states $s'\notin S$. In model-free RL, the former may
correspond to transitions available in a replay buffer, while the latter may be supplied by a value estimator.

Let $\pi^{*,sub}_S$ denote an optimal policy of $M^{sub}_{S,\pi}$.
We extend this policy to the original MDP as
\begin{align}
    \bar{\pi}(a|s)
    :=
    \begin{cases}
        \pi^{*,sub}_S(a|s), & s\in S,\\
        \pi(a|s), & s\notin S.
    \end{cases}
    \label{eq:sub_mdp_extended_policy}
\end{align}
The corresponding policy-improvement operator is denoted by
$\I_{\text{solve}\,M^{sub}_{S,\pi}}$.

We can now state
Theorem~\ref{thm:solving_sub_mdp_is_lower_bound_opt} more formally.

\begin{theorem}
    \label{thm:solving_sub_mdp_is_lower_bound_opt_formal}
    Let
    \begin{align}
        \Pi_S(\pi)
        :=
        \left\{
            \pi'\in\Pi :
            \pi'(\cdot|s)=\pi(\cdot|s),
            \ \forall s\notin S
        \right\}.
    \end{align}
    Then, for every $s\in S$,
    \begin{align}
        \max_{a\in\A}Q^\pi(s,a)
        \;\leq\;
        V^{*,sub}_{S,\pi}(s)
        \;\leq\;
        V^{\bar{\pi}}(s)
        \;\leq\;
        \sup_{\pi'\in\Pi_S(\pi)}V^{\pi'}(s).
        \label{eq:sub_mdp_value_bounds}
    \end{align}
    If leaving $S$ precludes returning to $S$, then, for every $s\in S$,
\begin{align}
    V^{*,\mathrm{sub}}_{S,\pi}(s)
    =
    V^{\bar\pi}(s)
    =
    \sup_{\pi'\in\Pi_S(\pi)} V^{\pi'}(s).
    \label{eq:sub_mdp_no_return_equivalence}
\end{align}
Thus, under this condition, solving $M^{\mathrm{sub}}_{S,\pi}$ is
equivalent to optimal policy improvement constrained by $S$.
    Consequently, since
    $\sup_{\pi'\in\Pi_S(\pi)}V^{\pi'}(s)\leq V^*(s)$,
    \begin{align}
        \left|
            V^*(s) - \max_{a\in\A}Q^\pi(s,a)
        \right|
        \geq
        \left|
            V^*(s) - V^{\I_{\text{solve}\,M^{sub}_{S,\pi}}(\pi)}(s)
        \right|
        \geq
        \inf_{\pi'\in\Pi_S(\pi)}
        \left|
            V^*(s) - V^{\pi'}(s)
        \right|.
        \label{eq:sub_mdp_optimality_gap_bounds}
    \end{align}
    Thus, solving $M^{sub}_{S,\pi}$ yields an improvement that is bounded
    below by one-step greedification $\max_a Q^\pi(s,a) $ and above by optimal policy improvement under the constraint that the policy may be modified only on $S$.
\end{theorem}

\textit{Proof.}
First observe that restricting the original policy $\pi$ to $S$ yields
exactly its original value on the sub-MDP. Indeed, for every $s\in S$,
\begin{align}
    V^\pi(s)
    &=
    \sum_{a\in\A}\pi(a|s)
    \Bigg[
        R(s,a)
        +
        \gamma
        \sum_{s'\in S}
        P(s'|s,a)V^\pi(s')
        +
        \gamma
        \sum_{s'\notin S}
        P(s'|s,a)V^\pi(s')
    \Bigg].
\end{align}
Hence $V^\pi|_S$ is the value of the policy $\pi|_S$ in
$M^{sub}_{S,\pi}$, and therefore
\begin{align}
    V^{*,sub}_{S,\pi}(s)\geq V^\pi(s),
    \qquad \forall s\in S.
    \label{eq:sub_mdp_dominates_pi}
\end{align}

Applying the optimality operator of the sub-MDP to $V^\pi$ gives
\begin{align}
    \left(\mathcal{T}^{sub}_{S,\pi}V^\pi\right)(s)
    &=
    \max_{a\in\A}
    \Bigg[
        R(s,a)
        +
        \gamma
        \sum_{s'\in S}
        P(s'|s,a)V^\pi(s')
        +
        \gamma
        \sum_{s'\notin S}
        P(s'|s,a)V^\pi(s')
    \Bigg]
    \\
    &=
    \max_{a\in\A}Q^\pi(s,a).
    \label{eq:sub_mdp_first_backup}
\end{align}
Since $\mathcal{T}^{sub}_{S,\pi}$ is monotone and
$V^{*,sub}_{S,\pi}$ is its optimal fixed point,
\begin{align}
    V^{*,sub}_{S,\pi}(s)
    \geq
    \left(\mathcal{T}^{sub}_{S,\pi}V^\pi\right)(s)
    =
    \max_{a\in\A}Q^\pi(s,a),
    \qquad \forall s\in S.
    \label{eq:sub_mdp_greedy_lower_bound}
\end{align}

We next relate the value of the sub-MDP to the value obtained by extending
its optimal policy to the original MDP according to
Equation~\ref{eq:sub_mdp_extended_policy}.
Define
\begin{align}
    W(s)
    :=
    \begin{cases}
        V^{*,sub}_{S,\pi}(s), & s\in S,\\
        V^\pi(s), & s\notin S.
    \end{cases}
\end{align}
For $s\in S$, optimality of $\pi^{*,sub}_S$ in the sub-MDP gives
\begin{align}
    \left(\mathcal{T}^{\bar{\pi}}W\right)(s)
    =
    W(s).
\end{align}
For $s\notin S$, $\bar{\pi}$ coincides with $\pi$. Moreover,
Equation~\ref{eq:sub_mdp_dominates_pi} implies $W\geq V^\pi$ everywhere.
By monotonicity of the Bellman operator,
\begin{align}
    \left(\mathcal{T}^{\bar{\pi}}W\right)(s)
    &=
    \left(\mathcal{T}^{\pi}W\right)(s)
    \\
    &\geq
    \left(\mathcal{T}^{\pi}V^\pi\right)(s)
    =
    V^\pi(s)
    =
    W(s),
    \qquad s\notin S.
\end{align}
Therefore,
\begin{align}
    \mathcal{T}^{\bar{\pi}}W \geq W.
\end{align}
Repeated application of the monotone Bellman operator
$\mathcal{T}^{\bar{\pi}}$, together with its convergence to
$V^{\bar{\pi}}$, yields
\begin{align}
    V^{\bar{\pi}}(s)\geq W(s),
\end{align}
and hence, in particular,
\begin{align}
    V^{\bar{\pi}}(s)
    \geq
    V^{*,sub}_{S,\pi}(s),
    \qquad \forall s\in S.
    \label{eq:extended_sub_policy_dominates_sub_value}
\end{align}
It remains to establish the equivalence when leaving $S$ precludes
returning to $S$. Consider any $\pi'\in\Pi_S(\pi)$. While the process
remains in $S$, the original MDP and $M^{\mathrm{sub}}_{S,\pi}$
have identical dynamics and rewards. If the process leaves $S$ at a
successor state $s'\notin S$, it follows $\pi$ thereafter, since
$\pi'$ and $\pi$ coincide outside $S$. By assumption, it never returns
to a state at which the two policies may differ. Its exact continuation
value is therefore $V^\pi(s')$, precisely the boundary value used by
$M^{\mathrm{sub}}_{S,\pi}$.

Consequently, for every $\pi'\in\Pi_S(\pi)$ and every $s\in S$,
\begin{align}
    V_M^{\pi'}(s)
    =
    V_{M^{\mathrm{sub}}_{S,\pi}}^{\pi'|_S}(s).
\end{align}
Optimizing both sides over $\pi'\in\Pi_S(\pi)$ gives
\begin{align}
    V^{*,\mathrm{sub}}_{S,\pi}(s)
    =
    \sup_{\pi'\in\Pi_S(\pi)}V^{\pi'}(s).
\end{align}
Since $\bar\pi|_S=\pi_S^{*,\mathrm{sub}}$, the same correspondence
gives
\begin{align}
    V^{\bar\pi}(s)
    =
    V^{*,\mathrm{sub}}_{S,\pi}(s),
\end{align}
which proves
Equation~\ref{eq:sub_mdp_no_return_equivalence}.

Finally, $\bar{\pi}\in\Pi_S(\pi)$ by construction. It follows immediately
that
\begin{align}
    V^{\bar{\pi}}(s)
    \leq
    \sup_{\pi'\in\Pi_S(\pi)}V^{\pi'}(s).
\end{align}
Combining this inequality with
Equations~\ref{eq:sub_mdp_greedy_lower_bound} and
\ref{eq:extended_sub_policy_dominates_sub_value} proves
Equation~\ref{eq:sub_mdp_value_bounds}, and therefore also
Equation~\ref{eq:sub_mdp_optimality_gap_bounds}.
\qed

The first inequality is strict at any state $s\in S$ for which
\begin{align}
    V^{*,sub}_{S,\pi}(s)
    >
    \left(\mathcal{T}^{sub}_{S,\pi}V^\pi\right)(s)
    =
    \max_{a\in\A}Q^\pi(s,a).
\end{align}
Thus, whenever a single greedy Bellman backup from $V^\pi$ does not attain
the optimal sub-MDP value, solving the sub-MDP can provide a strictly
stronger improvement than single-step greedification followed by followed by $\pi.$

\subsection{
Even when \texorpdfstring{$|S| = 1$}{|S| = 1}, greedification is not generally optimal policy improvement
}
\label{app:example_size_S_1}
Perhaps surprisingly, the greedy operator is \textit{not} an optimal policy improvement operator in general even when $|S| = 1$.
This arises from the fact that the greedy operator does not account for the benefits of taking better decisions $\pi'(s)$ at $s$. 
We will show this with an example.
If the agent cannot return from $s$ back to $s$ in the MDP (e.g. in non-cyclical MDPs), this is not a problem and indeed in this case the greedy operator is an optimal policy improvement operator on $S = \{s\}.$

Consider the following Markov Decision Process (MDP): 
$ \S = \{s, s_{terminal}\}, \A = \{a_0,a_1,a_2\}$.
Discount factor: $\gamma = 0.9.$
Deterministic rewards: $R(s,a_0) = 0, R(s,a_1) = 10, R(s,a_2) = 9 $.
Deterministic transitions: $ P(s_{terminal}|s,a_0) = 1, P(s_{terminal}|s,a_1) = 1, P(s|s,a_2) = 1 $ and otherwise $0$.
The optimal policy in this MDP $\pi^*(s) $ always chooses $a_2$ yielding a return of $\sum_{t = 0}^\infty 0.9^t \cdot 9 = \frac{9}{1 - 0.9} = 90$.

Let the current policy $\pi(s)$ be a policy that always chooses $a_0$.
Therefore, $V^\pi(s) = R(s,a_0) + \gamma V^\pi(s_{terminal}) = 0.  $
Let us compute $Q^\pi(s,\cdot):$
\begin{align}
    Q^\pi(s, a_0) &= 0 + 0.9 V^\pi(s_{\text{terminal}}) = 0 + 0 = 0, \\ 
    Q^\pi(s, a_1) &= 10 + 0.9  V^\pi(s_{\text{terminal}}) = 10 + 0 = 10, \\
    Q^\pi(s, a_2) &= 9 + 0.9  V^\pi(s) = 9 + 0.9\cdot 0 = 9.
\end{align}
If we use the greedy operator to improve the policy, we have:
\begin{align}
    \pi'(s) = \argmax_{a \in \A} Q^\pi(s,a), \quad P_{\pi'(s)}(a = a_1) = 1, \quad V^{\pi'}(s) = 10.
\end{align}
Which is not the optimal policy. 
In this MDP, \textit{optimal policy improvement} will always select $  a_2 $ at state $s.$

We conclude that if the MDP is cyclical, the greedy operator is \textit{not guaranteed} to provide optimal policy improvement \textit{even in an MDP where the policy is permitted to be modified only at one state}. 

\subsection{Optimal greedification under exact evaluation}
\label{app:theory:greedy_is_opt}

\begin{lemma}[Optimal greedification under exact evaluation]
\label{lem:exact_optimal_greedification}
Fix a policy $\pi\in\Pi$ and a state $s\in\mathcal S$. If
$\widehat Q^\pi=Q^\pi$, then any distribution
$\mu^*(\cdot\mid s)$ supported on
$\arg\max_{a\in\mathcal A}Q^\pi(s,a)$ satisfies
\[
\mu^*
\in
\arg\max_{\mu\in\Delta(\mathcal A)}
\sum_{a\in\mathcal A}\mu(a\mid s)Q^\pi(s,a)
=
\arg\max_{\mu\in\Delta(\mathcal A)}
\sum_{a\in\mathcal A}\mu(a\mid s)A^\pi(s,a).
\]
Consequently, $\mathcal I_{\arg\max}$ is an optimal greedification
operator under exact evaluation.
\end{lemma}

\begin{proof}
Both objectives are linear in $\mu$. Moreover,
\[
\sum_a \mu(a\mid s)A^\pi(s,a)
=
\sum_a \mu(a\mid s)Q^\pi(s,a)-V^\pi(s),
\]
and $V^\pi(s)$ is constant with respect to $\mu$. Hence the two
objectives have the same maximizers, namely the distributions supported
on $\arg\max_a Q^\pi(s,a)$.
\end{proof}

\begin{lemma}[When optimal greedification is optimal one-state improvement]
\label{lem:no_return_optimal_improvement}
Fix $\pi\in\Pi$ and $s\in\mathcal S$, and define
\[
\Pi_{\{s\}}(\pi)
:=
\left\{
\pi'\in\Pi:
\pi'(\cdot\mid x)=\pi(\cdot\mid x),
\ \forall x\neq s
\right\}.
\]
Assume that, after leaving $s$, the process cannot revisit $s$ under
the continuation policy $\pi$. Then, for every
$\pi'\in\Pi_{\{s\}}(\pi)$,
\[
V^{\pi'}(s)-V^\pi(s)
=
\sum_{a\in\mathcal A}\pi'(a\mid s)A^\pi(s,a).
\]
Therefore,
\[
\arg\max_{\pi'\in\Pi_{\{s\}}(\pi)}V^{\pi'}(s)
=
\arg\max_{\pi'\in\Pi_{\{s\}}(\pi)}
\sum_{a\in\mathcal A}\pi'(a\mid s)Q^\pi(s,a),
\]
and $\mathcal I_{\arg\max}$ is also an optimal policy-improvement
operator within $\Pi_{\{s\}}(\pi)$.
\end{lemma}

\begin{proof}
Because $s$ cannot be revisited, after the initial action the policies
$\pi'$ and $\pi$ coincide on every state that can subsequently be
visited. Hence
\[
V^{\pi'}(s)
=
\sum_{a\in\mathcal A}\pi'(a\mid s)Q^\pi(s,a).
\]
Subtracting
$V^\pi(s)=\sum_a\pi(a\mid s)Q^\pi(s,a)$ gives the first identity.
The remaining claims follow by maximizing the resulting linear
objective.
\end{proof}

The second lemma does not generally hold when $s$ can be revisited, as
demonstrated by the counterexample in Appendix~A.3.

\subsection{Greedification}
\label{proof:greedification}
\begin{proof}
Write $w(s,a) := 1/\sigma^2(s,a) = \tau^2 / \big((\tau^2 + \sigma^2_A(s,a))\beta\big)$, which is
strictly positive for every $a$ because $\beta > 0$, $\tau^2 > 0$ and $\sigma^2_A(s,a) \geq 0$.
Let $f(a) := \exp\big(w(s,a)A(s,a)\big) > 0$ and $Z := \sum_{b \in \A}\pi(b|s)f(b) > 0$, so that
$\I_{opt}(\pi,Q)(a|s) = \pi(a|s)f(a)/Z$.

Because $A(s,a) = Q(s,a) - \sum_{b\in\A}\pi(b|s)Q(s,b)$, the advantage satisfies
\begin{align}
    \sum_{a \in \A}\pi(a|s)A(s,a) = 0.
    \label{eq:adv_sums_to_zero}
\end{align}
Subtracting a constant from $Q(s,\cdot)$ shifts both sides of Definition \ref{def:go}
equally, so it suffices to show $\sum_{a}\I_{opt}(\pi,Q)(a|s)A(s,a) \geq 0$.

Compare $\pi(a|s)f(a)A(s,a)$ with $\pi(a|s)A(s,a)$ term by term. Since $w(s,a) > 0$:
if $A(s,a) > 0$ then $f(a) > 1$ and hence $f(a)A(s,a) > A(s,a)$; if $A(s,a) < 0$ then
$f(a) < 1$ and hence $f(a)A(s,a) > A(s,a)$; and if $A(s,a) = 0$ both sides vanish.
In every case $f(a)A(s,a) \geq A(s,a)$, strictly whenever $A(s,a) \neq 0$. Multiplying by
$\pi(a|s) \geq 0$, summing, and applying \eqref{eq:adv_sums_to_zero},
\begin{align}
    \sum_{a\in\A}\pi(a|s)f(a)A(s,a) \;\geq\; \sum_{a\in\A}\pi(a|s)A(s,a) \;=\; 0 .
\end{align}
Dividing by $Z > 0$ gives $\sum_{a}\I_{opt}(\pi,Q)(a|s)A(s,a) \geq 0$, with strict inequality as
soon as some $a$ with $\pi(a|s) > 0$ has $A(s,a) \neq 0$, i.e.\ unless $\pi$ is already
indifferent among the actions it plays. \qed
\end{proof}

\begin{remark}[why the advantage, and not $Q$, is essential]
\label{rem:shift_invariance}
For $\I_{reg}$ the substitution $A \leftrightarrow Q$ is harmless: a single scalar temperature
makes the softmax shift-invariant. For $\I_{opt}$ it is not. With per-action weights $w(s,a)$,
adding a constant $c$ to every $Q(s,a)$ changes the operator, because
$\exp\big(w(s,a)(Q(s,a)+c)\big) \propto \exp\big(w(s,a)c\big)$ is not a common factor. The proof above uses Equation \ref{eq:adv_sums_to_zero} and therefore genuinely requires the advantage;
implementations must centre by $V^\pi$ before applying the weights.

The following counter-example makes the point concrete and shows that the $Q$-based variant is
not a greedification operator. Take $\A = \{a_1,a_2\}$, $Q(s,a_1) = 1$, $Q(s,a_2) = 2$,
$\pi = [0.3, 0.7]$, and weights $w(s,a_1) = 10$, $w(s,a_2) = 1$. Here $V^\pi(s) = 1.7$ and
$A(s,\cdot) = [-0.7, +0.3]$. The advantage-based operator gives
\begin{align}
    \pi' \propto [\,0.3e^{-7},\; 0.7e^{0.3}\,] \approx [0.0003,\, 0.9997],
    \qquad \textstyle\sum_a \pi'(a|s)Q(s,a) \approx 2.00 > 1.7,
\end{align}
so greedification holds, as guaranteed above. Substituting $Q$ for $A$ gives instead
\begin{align}
    \pi' \propto [\,0.3e^{10},\; 0.7e^{2}\,] \approx [0.999,\, 0.001],
    \qquad \textstyle\sum_a \pi'(a|s)Q(s,a) \approx 1.00 < 1.7,
\end{align}
which moves mass onto the \emph{worse} action.
\end{remark}

\subsection{Convergence}
\label{proof:convergence}

We prove convergence of Generalized Policy Iteration
(Algorithm~\ref{alg:gpi}) driven by $\I_{opt}$ using the finite-horizon
analysis of \citet{tsitsiklis2002convergence} and the machinery developed by
\citet{viac}. In particular, the convergence proof of \citet{viac} relies on
the \emph{Limit-Sufficient Greedification} condition (Definition~5 therein).
We first show that $\I_{opt}$ satisfies the corresponding condition almost
surely in the presence of fresh, unbiased Gaussian evaluation noise whose
variance vanishes over iterations.

\subsubsection{Limit-Sufficient Greedification}
\label{proof:limit_sufficient_greedification}

\begin{lemma}[Limit-Sufficient Greedification under Vanishing Evaluation Noise]
    \label{thm:sufficient_greedification}
    Let $Q_0,Q_1,\dots\in\Q$ be a sequence of bounded functions such that,
    for every $s\in\S$ and $a\in\A$,
    \begin{align}
        Q_i(s,a)
        \xrightarrow[i\to\infty]{\mathrm{a.s.}}
        Q(s,a)
    \end{align}
    for some $Q\in\Q$.

    Let $\mathcal{F}_i$ denote the history prior to drawing the evaluation
    noise at iteration $i$. Assume that $Q_i$ and $\pi_i$ are
    $\mathcal{F}_i$-measurable and that, conditional on $\mathcal{F}_i$,
    the evaluation errors are fresh, independent zero-mean Gaussian random
    variables,
    \begin{align}
        \epsilon_i(s,a)
        \mid \mathcal{F}_i
        \sim
        \N\left(0,\sigma^2_{Q_i}(s,a)\right),
    \end{align}
    where
    \begin{align}
        \lim_{i\to\infty}\sigma^2_{Q_i}(s,a)=0,
        \qquad
        \forall s\in\S,\ a\in\A.
    \end{align}
    Let
    \begin{align}
        \hat Q_i(s,a)
        :=
        Q_i(s,a)+\epsilon_i(s,a),
    \end{align}
    and let
    \begin{align}
        \pi_{i+1}
        =
        \I_{opt}(\pi_i,\hat Q_i).
    \end{align}
    Assume that $\A$ is finite and that $\pi_0$ has full support over $\A$.

    Then, for every state $s\in\S$, if
    \begin{align}
        \A^*(s)
        :=
        \argmax_{a\in\A}Q(s,a),
    \end{align}
    the policy asymptotically concentrates on the maximizing action set:
    \begin{align}
        \sum_{a\in\A^*(s)}
        \pi_i(a|s)
        \xrightarrow[i\to\infty]{\mathrm{a.s.}}
        1.
        \label{eq:iopt_mass_concentration}
    \end{align}
    Consequently,
    \begin{align}
        \sum_{a\in\A}
        \pi_i(a|s)Q_i(s,a)
        \xrightarrow[i\to\infty]{\mathrm{a.s.}}
        \max_{a\in\A}Q(s,a),
        \label{eq:iopt_limit_sufficiency}
    \end{align}
    which is the Limit-Sufficient Greedification condition of
    \citet{viac}, extended here to vanishing stochastic evaluation error.
\end{lemma}

\begin{proof}
Fix a state $s\in\mathcal S$ and suppress $s$ from the notation where convenient.
Define the true and noisy advantages at iteration $i$ by
\begin{align}
    A_i(a)
    &:=
    Q_i(s,a)
    -
    \sum_{b\in\mathcal A}
        \pi_i(b\mid s)Q_i(s,b),
    \\
    \hat A_i(a)
    &:=
    \hat Q_i(s,a)
    -
    \sum_{b\in\mathcal A}
        \pi_i(b\mid s)\hat Q_i(s,b).
\end{align}
Let
\begin{align}
    v_i(a)
    &:=
    \sigma_{\hat A_i}^2(s,a)
    \\
    &:=
    \bigl(1-\pi_i(a\mid s)\bigr)^2
        \sigma_{Q_i}^2(s,a)
    +
    \sum_{b\neq a}
        \pi_i(b\mid s)^2
        \sigma_{Q_i}^2(s,b),
\end{align}
and denote the denominator in the exponent of $\mathcal I_{\mathrm{opt}}$ by
\begin{align}
    d_i(a)
    :=
    \beta\frac{\tau^2+v_i(a)}{\tau^2}.
\end{align}
Because
\begin{align}
    0
    \leq
    v_i(a)
    \leq
    \sum_{b\in\mathcal A}
        \sigma_{Q_i}^2(s,b),
\end{align}
the finiteness of $\mathcal A$ and the assumption
$\sigma_{Q_i}^2(s,b)\to0$ imply
\begin{align}
    d_i(a)\to\beta
\end{align}
uniformly over $a\in\mathcal A$. Moreover, $d_i(a)\geq\beta>0$ for every
$i$ and $a$.

The policy update can now be written as
\begin{align}
    \pi_{i+1}(a\mid s)
    \propto
    \pi_i(a\mid s)
    \exp\left(
        \frac{\hat A_i(a)}{d_i(a)}
    \right).
\end{align}
Because $\pi_0$ has full support and the multiplicative factor is strictly
positive, $\pi_i$ has full support for every finite $i$.

Fix an optimal action $a^\star\in\mathcal A^\star(s)$ and a suboptimal
action $a\notin\mathcal A^\star(s)$, and define their limiting action gap
\begin{align}
    \Delta_a
    :=
    Q(s,a^\star)-Q(s,a)
    >
    0.
\end{align}
Taking the ratio of their probabilities eliminates the normalization
constant:
\begin{align}
    \log
    \frac{\pi_{i+1}(a^\star\mid s)}
         {\pi_{i+1}(a\mid s)}
    &=
    \log
    \frac{\pi_i(a^\star\mid s)}
         {\pi_i(a\mid s)}
    +
    Y_i,
    \\
    Y_i
    &:=
    \frac{\hat A_i(a^\star)}{d_i(a^\star)}
    -
    \frac{\hat A_i(a)}{d_i(a)}.
\end{align}

Decompose this increment into its conditional mean and a zero-mean
remainder:
\begin{align}
    g_i
    &:=
    \mathbb E[Y_i\mid\mathcal F_i],
    &
    \xi_i
    &:=
    Y_i-g_i.
\end{align}
Because the evaluation errors have conditional mean zero,
\begin{align}
    g_i
    =
    \frac{A_i(a^\star)}{d_i(a^\star)}
    -
    \frac{A_i(a)}{d_i(a)}.
\end{align}
Writing
\begin{align}
    B_i
    :=
    \sum_{b\in\mathcal A}
        \pi_i(b\mid s)Q_i(s,b),
\end{align}
we equivalently have
\begin{align}
    g_i
    =
    \frac{Q_i(s,a^\star)}{d_i(a^\star)}
    -
    \frac{Q_i(s,a)}{d_i(a)}
    +
    B_i
    \left(
        \frac{1}{d_i(a)}
        -
        \frac{1}{d_i(a^\star)}
    \right).
\end{align}
Since $Q_i(s,\cdot)\to Q(s,\cdot)$ over the finite action set, the
sequence $\{B_i\}$ is bounded. Together with the uniform convergence
$d_i(\cdot)\to\beta$, this gives
\begin{align}
    g_i
    \xrightarrow[i\to\infty]{\mathrm{a.s.}}
    \frac{Q(s,a^\star)-Q(s,a)}{\beta}
    =
    \frac{\Delta_a}{\beta}
    >
    0.
\end{align}
Consequently, by the elementary Ces\`aro lemma,
\begin{align}
    \frac{1}{n}
    \sum_{i=0}^{n-1}g_i
    \xrightarrow[n\to\infty]{\mathrm{a.s.}}
    \frac{\Delta_a}{\beta}.
\end{align}

By construction,
$\mathbb E[\xi_i\mid\mathcal F_i]=0$, so $\{\xi_i\}$ is a
martingale-difference sequence. Each $\xi_i$ is a finite linear
combination of the evaluation errors. Because $d_i(a)\geq\beta>0$ and
the convergent variance sequences $\{\sigma_{Q_i}^2(s,b)\}_i$ are
bounded, there exists a finite constant $C$ such that
\begin{align}
    \mathbb E[\xi_i^2\mid\mathcal F_i]
    \leq C
    \qquad
    \forall i.
\end{align}
Hence,
\begin{align}
    \sum_{i=0}^{\infty}
    \frac{
        \mathbb E[\xi_i^2\mid\mathcal F_i]
    }{(i+1)^2}
    <
    \infty.
\end{align}
The strong law for martingale-difference sequences
\citep[Theorem~2.18]{hall1980martingale} therefore gives
\begin{align}
    \frac{1}{n}
    \sum_{i=0}^{n-1}\xi_i
    \xrightarrow[n\to\infty]{\mathrm{a.s.}}
    0.
\end{align}

Unrolling the log-ratio recursion and dividing by $n$ now yields
\begin{align}
    \frac{1}{n}
    \log
    \frac{\pi_n(a^\star\mid s)}
         {\pi_n(a\mid s)}
    &=
    \frac{1}{n}
    \log
    \frac{\pi_0(a^\star\mid s)}
         {\pi_0(a\mid s)}
    +
    \frac{1}{n}
    \sum_{i=0}^{n-1}g_i
    +
    \frac{1}{n}
    \sum_{i=0}^{n-1}\xi_i
    \\
    &\xrightarrow[n\to\infty]{\mathrm{a.s.}}
    \frac{\Delta_a}{\beta}
    >
    0.
\end{align}
Therefore,
\begin{align}
    \frac{\pi_n(a\mid s)}
         {\pi_n(a^\star\mid s)}
    \xrightarrow[n\to\infty]{\mathrm{a.s.}}
    0.
\end{align}
Since $\pi_n(a^\star\mid s)\leq1$, this implies
\begin{align}
    \pi_n(a\mid s)
    \xrightarrow[n\to\infty]{\mathrm{a.s.}}
    0
    \qquad
    \forall a\notin\mathcal A^\star(s).
\end{align}
Because $\mathcal A$ is finite,
\begin{align}
    \sum_{a\in\mathcal A^\star(s)}
        \pi_n(a\mid s)
    \xrightarrow[n\to\infty]{\mathrm{a.s.}}
    1,
\end{align}
which proves the first claim.

Finally, let $Q_{\max}(s):=\max_a Q(s,a)$. Then
\begin{align}
    0
    &\leq
    Q_{\max}(s)
    -
    \sum_{a\in\mathcal A}
        \pi_i(a\mid s)Q(s,a)
    \\
    &=
    \sum_{a\notin\mathcal A^\star(s)}
        \pi_i(a\mid s)
        \bigl(Q_{\max}(s)-Q(s,a)\bigr)
    \xrightarrow[i\to\infty]{\mathrm{a.s.}}
    0.
\end{align}
Moreover,
\begin{align}
    \left|
        \sum_{a\in\mathcal A}
            \pi_i(a\mid s)Q_i(s,a)
        -
        Q_{\max}(s)
    \right|
    &\leq
    \lVert Q_i(s,\cdot)-Q(s,\cdot)\rVert_\infty
    \\
    &\quad+
    \left|
        \sum_{a\in\mathcal A}
            \pi_i(a\mid s)Q(s,a)
        -
        Q_{\max}(s)
    \right|
    \xrightarrow[i\to\infty]{\mathrm{a.s.}}
    0.
\end{align}
This is the required Limit-Sufficient Greedification condition.
\end{proof}

\subsubsection{Convergence using Limit-Sufficient Greedification}
\label{proof:convergence_final_part}

\begin{proof}
We extend the Generalized Policy Iteration setup of
Algorithm~\ref{alg:gpi} by allowing the evaluation supplied to the policy
improvement operator at iteration $i$ to be corrupted according to
\begin{align}
    \hat Q_i(s,a)
    =
    Q_i(s,a)+\epsilon_i(s,a),
\end{align}
where the errors satisfy the assumptions of
Lemma~\ref{thm:sufficient_greedification}.
The underlying evaluation update of Algorithm~\ref{alg:gpi} remains
unchanged; the additional error affects the evaluation supplied to the
policy-improvement operator.

The convergence proof of Theorem~3 in \citet{viac} proceeds by induction
backwards from terminal states in a finite-horizon MDP. Its
Limit-Sufficient Greedification step requires that, whenever the sequence
of action-value estimates at a state converges to some limiting function
$Q$, the corresponding policy sequence satisfies
\begin{align}
    \sum_{a\in\A}
    \pi_i(a|s)Q_i(s,a)
    \longrightarrow
    \max_{a\in\A}Q(s,a).
    \label{eq:noisy_limit_sufficiency_final}
\end{align}
Lemma~\ref{thm:sufficient_greedification} establishes exactly
Equation~\ref{eq:noisy_limit_sufficiency_final} almost surely under the
additional evaluation noise considered here.

Condition on the probability-one event on which
Equation~\ref{eq:noisy_limit_sufficiency_final} holds at every state.
The finite-horizon induction of \citet{viac} then applies unchanged:
convergence at terminal states implies convergence of the action values at
the preceding time step; Limit-Sufficient Greedification implies that the
policy becomes greedy with respect to those limiting action values; and the
argument proceeds recursively backwards through the horizon.
Since the state and action spaces and the horizon are finite, the
intersection of the corresponding probability-one events is itself a
probability-one event.

Therefore, Generalized Policy Iteration driven by $\I_{opt}$ converges
almost surely to the optimal value:
\begin{align}
    Q_i
    \xrightarrow[i\to\infty]{\mathrm{a.s.}}
    Q^*,
    \qquad
    V^{\pi_i}
    \xrightarrow[i\to\infty]{\mathrm{a.s.}}
    V^*.
\end{align}
Moreover, at every state,
\begin{align}
    \sum_{a\in\argmax_b Q^*(s,b)}
    \pi_i(a|s)
    \xrightarrow[i\to\infty]{\mathrm{a.s.}}
    1.
\end{align}
Thus, the policy converges to the set of optimal policies in the sense that
all probability mass assigned to strictly suboptimal actions vanishes.
If every state has a unique optimal action, then $\pi_i$ itself converges
almost surely to the corresponding optimal policy.
\end{proof}
\section{Derivations and Discussion}
\label{app:derivations}

\subsection{Derivation of the advantage posterior in the likelihood}
\label{app:derivations:likelihood}

For a fixed state $s$, we model the approximate evaluation as
\begin{align}
    \hat Q^\pi(s,a)
    =
    Q^\pi(s,a)+\epsilon(s,a),
    \qquad
    \epsilon(s,a)
    \sim
    \mathcal N\!\left(
        0,
        \sigma_{\hat Q^\pi}^2(s,a)
    \right),
\end{align}
with the errors independent across actions.

The corresponding approximate advantage is
\begin{align}
    \hat A^\pi(s,a)
    &:=
    \hat Q^\pi(s,a)
    -
    \sum_{b\in\mathcal A}
    \pi(b\mid s)\hat Q^\pi(s,b)
    \nonumber\\
    &=
    A^\pi(s,a)
    +
    \bigl(1-\pi(a\mid s)\bigr)\epsilon(s,a)
    -
    \sum_{b\neq a}
    \pi(b\mid s)\epsilon(s,b).
    \label{eq:appendix_advantage_error}
\end{align}
Because Equation~\ref{eq:appendix_advantage_error} is an affine
transformation of independent Gaussian errors,
\begin{align}
    \hat A^\pi(s,a)
    \mid
    A^\pi(s,a),\pi,s
    \sim
    \mathcal N\!\left(
        A^\pi(s,a),
        \sigma_{\hat A^\pi}^2(s,a)
    \right),
\end{align}
where
\begin{align}
    \sigma_{\hat A^\pi}^2(s,a)
    &:=
    \bigl(1-\pi(a\mid s)\bigr)^2
    \sigma_{\hat Q^\pi}^2(s,a)
    +
    \sum_{b\neq a}
    \pi(b\mid s)^2
    \sigma_{\hat Q^\pi}^2(s,b).
    \label{eq:appendix_advantage_variance}
\end{align}

For analytical tractability, we place the action-wise prior
\begin{align}
    A^\pi(s,a)\sim\mathcal N(0,\tau^2).
\end{align}
Taken jointly across actions, the corresponding product prior ignores
the exact constraint
\[
    \sum_{a\in\mathcal A}
    \pi(a\mid s)A^\pi(s,a)=0.
\]
We make this modeling approximation because it yields a tractable
action-wise posterior and expresses that, absent evaluation evidence,
the model does not favor any action over another.

By Bayes' rule,
\begin{align}
    &p\!\left(
        A^\pi(s,a)
        \mid
        \hat A^\pi(s,a),\pi,s
    \right)
    \nonumber\\
    &\quad \propto
    \exp\left[
        -\frac{1}{2}
        \left(
            \frac{(A^\pi(s,a))^2}{\tau^2}
            +
            \frac{
                \bigl(
                    \hat A^\pi(s,a)-A^\pi(s,a)
                \bigr)^2
            }{
                \sigma_{\hat A^\pi}^2(s,a)
            }
        \right)
    \right].
\end{align}
Completing the square gives
\begin{align}
    p\!\left(
        A^\pi(s,a)
        \mid
        \hat A^\pi(s,a),\pi,s
    \right)
    =
    \mathcal N\!\left(
        m^\pi(s,a),
        v^\pi(s,a)
    \right),
\end{align}
with
\begin{align}
    m^\pi(s,a)
    &=
    \frac{\tau^2}{
        \tau^2+\sigma_{\hat A^\pi}^2(s,a)
    }
    \hat A^\pi(s,a),
    \\
    v^\pi(s,a)
    &=
    \frac{
        \tau^2\sigma_{\hat A^\pi}^2(s,a)
    }{
        \tau^2+\sigma_{\hat A^\pi}^2(s,a)
    }.
\end{align}

This posterior is exact for the scalar action-wise observation model.
In the main text, we approximate the joint posterior over
$A^\pi(s,\cdot)$ by the product of these action-wise posteriors.

\subsection{Modeling the Utility Factor using Rational Choice and Expected Utility Theory}
\label{app:rational_choice_theory_backround}
Rational Choice Theory is a foundational framework widely popular in the fields of economics, political science, and behavioral sciences for modeling complex, goal-oriented decision-making \citep{becker1976economic}.
Rational Choice Theory posits that decision-makers consistently select actions that maximize their subjective preferences \citep{arrow2012social}. 
Under conditions of uncertainty, this is formalized by Expected Utility Theory (EUT), which demonstrates that a rational agent's preferences can be mathematically represented by the expectation of a scalar utility function over possible outcomes \citep{von1947theory, savage2012foundations}. 
To operationalize EUT within probabilistic frameworks, deterministic utility maximization is standardly relaxed into probabilistic choice models. 
Rooted in Luce's choice axiom \citep{luce1959individual} and random utility theory \citep{mcfadden1973conditional}, these models construct likelihood functions where the probability of selecting an alternative is proportional to its exponentiated utility. 

This approach allows us to model \textit{policy optimization} as reasoning over which policy $\pi'$ is preferred based on its expected utility $A^{\pi',\pi}(s)$.
We model the policy's \textit{expected utility for improvement} $A^{\pi',\pi}(s)$ with its local improvement surrogate over $\pi$.
The derived Bayesian posterior $p(A^\pi(s,a) | \hat A^\pi(s,a), \pi, s)$ for the uncertain advantage allows us to account for the evaluation uncertainty within policy optimization and yields the tractable certainty-equivalent objective in Equation \ref{eq:ce_decision_objective}, for which $\I_{\mathrm{opt}}$ is the exact MAP optimizer.

\subsection{The utility as the local improvement of $\pi'$ over $\pi$}
\label{app:derivation_of_the_objective}
We have:
\begin{align}
    A^{\pi',\pi}(s) 
    &= \sum_{a \in A}\pi'(a|s)A^\pi(s,a)  \\
    &= \sum_{a \in A}\pi'(a|s)(Q^\pi(s,a)-V^\pi(s))  \\
    &= \sum_{a \in A}\pi'(a|s)Q^\pi(s,a)-\sum_{a \in A}\pi'(a|s)V^\pi(s) \\
    &= \sum_{a \in A}\pi'(a|s)Q^\pi(s,a)-V^\pi(s)\sum_{a \in A}\pi'(a|s) \\
    &= \sum_{a \in A}\pi'(a|s)Q^\pi(s,a)-V^\pi(s) \\
    &= V^{\pi'}(s) -V^\pi(s). \label{eq:value_of_improved_policy}
\end{align}
Equation \ref{eq:value_of_improved_policy} holds when \(\pi'\) differs from \(\pi\) only at \(s\) and \(s\) cannot be revisited after leaving it. 
In particular, this holds for a one-state policy update in a non-cyclical MDP.
However, we always have:
\begin{align}
    \argmax_{\pi'} A^{\pi',\pi}(s) 
    = \argmax_{\pi'} \sum_{a \in A}\pi'(a|s)A^\pi(s,a)
    = \argmax_{\pi'} \sum_{a \in A}\pi'(a|s)Q^\pi(s,a)
\end{align}
because $V^\pi(s)$ does not depend on $a$.
In other words, this objective is always sound as it shares the same maximizer as $ \argmax_{\pi'} \sum_{a \in A}\pi'(a|s)Q^\pi(s,a). $

We use $ A^{\pi',\pi}(s) = \sum_{a \in A}\pi'(a|s)A^\pi(s,a) $ instead of $\argmax_{\pi'} \sum_{a \in A}\pi'(a|s)Q^\pi(s,a)$, because the advantage has important properties for policy improvement.
Specifically, it is \textit{informed} with respect to $\pi$, guaranteeing that the sign is always positive $A^\pi(s,a) > 0$ for actions that are \textit{better} than those chosen by the policy (e.g. \textit{improved}) and negative $A^\pi(s,a) < 0$ for actions that are \textit{worse}.

\subsection{Derivation of the certainty-equivalent MAP objective}
\label{app:derivations:map}
Writing $\beta_0$ for the coefficient of the policy prior in
Equation~\ref{eq:policy_prior}, the certainty-equivalent posterior is
\begin{align}
    p_{\mathrm{CE}}\!\left(
        \pi' \mid \hat A^\pi(s,\cdot),\pi,s
    \right)
    &\propto
    \exp\left(
        \lambda
        \sum_{a\in\mathcal A}
        \pi'(a\mid s)
        \frac{\tau^2}{
            \tau^2+\sigma_{\hat A^\pi}^2(s,a)
        }
        \hat A^\pi(s,a)
    \right)
    \nonumber\\
    &\quad\times
    \exp\left(
        -\beta_0
        D_{\mathrm{KL}}\!\left(
            \pi'(\cdot\mid s)\,\|\,\pi(\cdot\mid s)
        \right)
    \right).
\end{align}
Therefore, because $\lambda>0$,
\begin{align}
    \arg\max_{\pi'}\,
    p_{\mathrm{CE}}\!\left(
        \pi' \mid \hat A^\pi(s,\cdot),\pi,s
    \right)
    &=
    \arg\max_{\pi'}
    \left\{
        \sum_{a\in\mathcal A}
        \pi'(a\mid s)
        \frac{\tau^2}{
            \tau^2+\sigma_{\hat A^\pi}^2(s,a)
        }
        \hat A^\pi(s,a)
        \right.
        \nonumber\\[-0.25em]
    &\hspace{5em}\left.
        -
        \frac{\beta_0}{\lambda}
        D_{\mathrm{KL}}\!\left(
            \pi'(\cdot\mid s)\,\|\,\pi(\cdot\mid s)
        \right)
    \right\}.
\end{align}
Defining $\beta:=\beta_0/\lambda$ recovers the
certainty-equivalent objective in Equation~\ref{eq:ce_decision_objective}.

\subsection{Derivation of the solution for the certainty-equivalent MAP objective}
\label{app:derivations:solution_to_map}

We derive $\I_{\text{opt}}$, the solution to the optimization problem:
\begin{align}
\I_{opt}(\pi,\hat Q^\pi)(s) = \argmax_{\pi'} \left(
\sum_{a} \pi'(a|s) w(s,a) \hat A^\pi(s,a)
- \beta \sum_{a} \pi'(a|s)\log \frac{\pi'(a|s)}{\pi(a|s)}
\right),
\end{align}
where $w(s,a) = \frac{\tau^2}{\tau^2 + \sigma_{\hat A^\pi}^2(s,a)}$, below.

Introduce a Lagrange multiplier $\nu$ for the constraint $\sum_a \pi'(a|s)=1$:
\begin{align}
\mathcal{L}
&= \sum_a \pi'(a|s) w(s,a) \hat A^\pi(s,a)
- \beta \sum_a \pi'(a|s)\log \frac{\pi'(a|s)}{\pi(a|s)}
+ \nu \left( \sum_a \pi'(a|s) - 1 \right).
\end{align}

Take derivative w.r.t. $\pi'(a|s)$ and set to zero:
\begin{align}
0
&= w(s,a) \hat A^\pi(s,a)
- \beta \left( \log \frac{\pi'(a|s)}{\pi(a|s)} + 1 \right)
+ \nu.
\end{align}

Rearranging,
\begin{align}
\log \frac{\pi'(a|s)}{\pi(a|s)}
&= \frac{1}{\beta} w(s,a) \hat A^\pi(s,a) + \frac{\nu}{\beta} - 1.
\end{align}

Exponentiating,
\begin{align}
\pi'(a|s)
&= \pi(a|s)\exp\left(\frac{1}{\beta} w(s,a) \hat A^\pi(s,a)\right)
\exp\left(\frac{\nu}{\beta} - 1\right).
\end{align}
To guarantee that $\pi'$ is a proper probability distribution $\sum_{a \in A}\pi'(a|s) = 1$, $\nu$ is chosen such that $\exp\left(\frac{\nu}{\beta} - 1\right)$ is the normalization constant:
\begin{align}
    \exp\left(\frac{\nu}{\beta}-1\right)
    =
    \left[
    \sum_b
    \pi(b|s)
    \exp\left(
        \frac{1}{\beta}w(s,b)\hat A^\pi(s,b)
    \right)
    \right]^{-1}.
\end{align}
Thus, we have:
\begin{align}
    \pi'(a|s)
    &= \frac{\pi(a|s)\exp\left(\frac{1}{\beta} w(s,a) \hat A^\pi(s,a)\right)}
    {\sum_{b} \pi(b|s)\exp\left(\frac{1}{\beta} w(s,b) \hat A^\pi(s,b)\right)}
    \\
    &\propto
    \pi(a|s) \exp \left(\frac{\tau^2}{(\tau^2 + \sigma_{\hat A^\pi}^2(s,a))\beta}\hat A^\pi(s,a)\right)
\end{align}

\subsection{Modeling evaluation uncertainty with the Normal distribution}
\label{app:uncertain_eval_as_gaussian_excuses}

In this section we further discuss the suitability of a Gaussian distribution to model the evaluation error, beyond its advantages for analytical tractability, and how we estimate the error in practice.
As discussed in Section \ref{sec:optimal_greedification},
\citet{hui2023doublegumbel} motivate Gaussian TD error through the Central Limit Theorem and empirically find a moment-matched heteroscedastic Gaussian to provide a suitable approximation throughout training. 
More generally, evaluation error in modern RL can arise from stochastic returns, bootstrapped targets, sampled actions and function approximation. 
These sources need not be independent or exactly Gaussian, so we use the heteroscedastic Gaussian as an approximation rather than an exact description of modern evaluation error. 

The better we can estimate the resulting state-action evaluation uncertainty, the better $\I_{\text{opt}}$ can be expected to perform.
Uncertainty estimation remains an active and major field of research for AI, ML and RL, and we expect the potential of $\I_{\text{opt}}$ to grow with the quality of the estimator.
To strengthen the significance of the contribution of $\I_{\text{opt}}$ to RL, in this work we incorporate $\I_{\text{opt}}$ into algorithms that already use some form of uncertainty estimation, and rely on this source - even when it is rudimentary.

In Policy Iteration, we noise the evaluation with Gaussian noise with known variance, and provide the agent with this exact variance.
As expected, $\I_{\text{opt}}$ is the best performer in this setting.

In model-based RL, we incorporate $\I_{\text{opt}}$ into MCTS.
Classical MCTS estimates action values from sample means of bounded Monte-Carlo returns
\citep{UCT}. 
Under fixed-distribution independent sampling, these means are unbiased, their variance decreases as $1/N(s,a)$, and their error becomes asymptotically Gaussian.
Modern DNN-based MCTS \citep{AlphaZero,muzero} departs substantially from this idealized sampling setting yet remains highly successful, suggesting that exact satisfaction of these assumptions is not necessary for the corresponding design principles to remain useful.
We therefore use $\sigma_{\hat Q^\pi}^2(s,a)\propto1/N(s,a)$ as a simple variance proxy and directly examine
the resulting DNN residuals in Figure~\ref{fig:evaluation_error}.

In model free RL (SAC and ReBRAC) we use the two-critic value ensemble. Value ensembles are a standard approach for value-uncertainty estimation in RL \citep{lakshminarayanan2017simple,gawlikowski2023survey} and as we see, they are sufficient to enable $\I_{\text{opt}}$ to provide improvement over the baseline, both in offline as well as online RL (Figures \ref{fig:flagship} and \ref{fig:sac_per_env}, and Table \ref{tab:offline_per_env}).

\subsection{\texorpdfstring{$\I_{opt}$}{I\_opt} and exploration}
\label{app:I_opt_and_exploration}

The operator $\I_{opt}$ can be expected to have an effect on the exploration of RL algorithms along multiple lines.
Due to the complexity of exploration as its own independent research field, we divide the discussion into three parts.
First, we will motivate why exploration does not need to be considered by $\I_{opt}$ at all.
Second, we will motivate what effects $\I_{opt}$ \textit{could} have on exploration.
Last, we will discuss how exploration should be incorporated into algorithms utilizing $\I_{opt}$.

\subsubsection{Why we need not worry about the effects \texorpdfstring{$\I_{opt}$}{I\_opt} has on exploration}
Although the simplest and most popular approach is to use trained \textit{exploitation} policy $\pi_\theta$ to act in the environment, and implicitly to explore, there is no principle necessitating that.
$\pi_\theta$ need not be used for exploration at all.
\cite{oren2025epistemic} showed that one could use in a principled manner a policy trained explicitly for exploration to benefit the exploitation policy learned by $\pi_\theta$.

Due to this principled separation between policies, we need not even consider explicitly the effects $\I_{opt}$ has on exploration - because the policy it is used to train $\pi_\theta$ need not be used for exploration at all.

\subsubsection{What effects \texorpdfstring{$\I_{opt}$}{I\_opt} \textit{could} have on exploration}
Of course, in practice, the simplest approach would be to replace the existing greedification operator used by an RL algorithm with $\I_{\text{opt}}$ and act with the trained exploitation policy $\pi_\theta$ in the environment (which is indeed how many popular, practical algorithms work in practice \citep{gumbel_mz,td_mpc2,sac1}).
Despite not being the only alternative, it is popular in practice for several reasons.
Not only is it simpler, but in complex environments such as Chess or Go it is not immediately clear which principled exploration policy is best suited, and acting with the trained policy is known to work well in practice \citep{AlphaZero,gumbel_mz}.
From this perspective, it may still be interesting to reason over what effects $\I_{opt}$ \textit{could} have on exploration.
We provide a detailed set of conjectures below:
\begin{enumerate}
    \item Over-estimation bias should reduce, because the operator is more stable, because it propagates less error, because it favors high-error values less. As a result of reduced over-estimation bias, exploration may reduce.
    \item During acting, exploration may also be reduced as a result of not favoring high-error, high-value actions.
    \item However, because \textit{high}-error, \textit{low}-value actions are less unfavored, exploration can overall \textit{increase}, as well.
\end{enumerate}
Due to these conflicting possible effects, it is hard to predict whether acting with $\pi_\theta$ trained with $\I_{opt}$ increase, decrease or have other effects on exploration in aggregate.
However, due to the last point in this list, we would expect $\I_{opt}$ in practice to simply induce a better learning behavior over algorithms that do not employ explicit directed exploration in the majority of domains.
This result is supported by the results of our experiments in online RL (GumbelAlphaZero and SAC).

\subsubsection{How to utilize \texorpdfstring{$\I_{opt}$}{I\_opt} for principled exploration}
As shown by \cite{oren2025epistemic}, a principled approach to training an exploration policy $\pi_e$ is using policy improvement operators.
However, instead of the improvement being with respect to $Q$ or $A$, it is with respect to an \textit{exploration statistic}.
A popular and principled exploration statistic is the \textit{maximum upper confidence bound} (UCB), to drive classical, principled UCB exploration.
This statistic can be modeled as $ U_M^\pi(s,a) = Q_M^\pi(s,a) + \beta \sqrt{\V_{epistemic}\,[Q_M^\pi]}$ where $ \V_{epistemic} $ here denotes the \textit{epistemic} uncertainty with respect to $Q_M^\pi$.
The random variable is the model $M$, and the uncertainty stems from the model not necessarily being aligned with reality.
We refer to \cite{oren2025epistemic} for more detail.

Similarly, $\I_{opt}$ can be used to maximize $ \argmax_{\pi'} \sum_{a\in\A}U^\pi(s,a)\pi'(a|s) $ in the same manner it is used to maximize $ \argmax_{\pi'} \sum_{a\in\A}Q^\pi(s,a)\pi'(a|s).$ 
We leave investigating the effect of such operator on the exploration policy to future work.
However, we can predict that the operator has the same underlying effect: actions with \textit{certain} high UCB will be favored, actions with \textit{certain} low UCB will be unfavored, and actions with \textit{uncertain} UCB will not be strongly modified in the exploration policy.

\section{Connections to prior work in more detail}
\label{sec:more_related_work}
\textbf{Control as Inference (CAI) and "Optimal Control"} \citep{levine2018cai,optimal_control}
formulate finding a policy that maximizes expected return as a probabilistic inference problem in a graphical model.
Trajectories are treated as latent variables generated by the dynamics $P$ and a prior policy 
$\pi$.
Rewards are incorporated as likelihood terms by introducing binary optimality variables $O_t$ such that $ p(O_t=1| s_t,a_t)\propto \exp(r(s_t,a_t)/\eta)$.
Inference in this model leads to the "soft-optimal"
state value function 
$ V^{\pi'}_{\text{soft}}(s) 
$
and action-value function 
$ Q^{\pi'}_{\text{soft}}(s,a):
$
\begin{align}
Q_{\mathrm{soft}}(s,a)
&=
r(s,a)+\gamma
\E_{s'\sim P(\cdot|s,a)}
[V_{\mathrm{soft}}(s')],
\\
V_{\mathrm{soft}}(s)
&=
\eta\log
\sum_a
\pi(a|s)
\exp\left(
    Q_{\mathrm{soft}}(s,a)/\eta
\right).
\end{align}
The so-called "soft-optimal" policy $\pi'$ is given by:
\begin{align*}
    \pi'(a|s) = \I_{\text{cai}} (\pi,Q^{\pi'})(a|s)
    = 
    \pi(a|s)\exp\big(\frac{Q^{\pi'}_{\text{soft}}(s,a) - V^{\pi'}_{\text{soft}}(s) }{\eta}\big)
    =
    \pi(a|s)\exp\big(\frac{A^{\pi'}_{\text{soft}}(s,a)}{\eta}\big).
\end{align*}
In practice, the agent does not have access to $Q_{\text{soft}}^{\pi'}, V_{\text{soft}}^{\pi'}$ and therefore $Q_{\text{soft}}^{\pi} \approx Q_{\text{soft}}^{\pi'}, V_{\text{soft}}^{\pi} \approx V_{\text{soft}}^{\pi'}$ is used.
This is equivalent to $\I_{\text{reg}}$ with $A^\pi$ (except for softness, which is often omitted in practice \citep{td_mpc2,trtpi}), and equivalently to $\I_{\text{opt}}$ with $\sigma_{\hat A^\pi}^2(s,a) = \sigma $, a notion of uniform uncertainty across all state-action pairs.
This is also the operator used by MPO \citep{mpo}.

Section \ref{sec:optimal_greedification} suggests that it is more accurate to interpret $\I_{\text{cai}}$ as producing a notion of \textit{optimal improvement} with respect to $\pi$ rather than as producing an approximation of the (soft) \textit{optimal policy}.
It also shows how this operator emerges naturally from modeling policy improvement under evaluation uncertainty, without necessitating the modified objective, soft-values and non-trivial additional complexity resulting from modeling RL as a probabilistic inference problem.

\textbf{Mirror Descent (MD) and The Natural Policy Gradient (NPG)} \citep{mirror_descent_paper,npg}
connect to $\I_{\mathrm{reg}}$ in the finite tabular setting with direct policy parameterization. 
For a KL-Bregman divergence, the exact PMD update with step size $\alpha_k>0$ has the state-wise form
\citep{mirror_descent_derivations}
\begin{align}
    \pi_{k+1}(a\mid s)
    &\propto
    \pi_k(a\mid s)
    \exp\left(
        \alpha_k Q^{\pi_k}(s,a)
    \right)
    \nonumber\\
    &\propto
    \pi_k(a\mid s)
    \exp\left(
        \alpha_k A^{\pi_k}(s,a)
    \right).
\end{align}
\citet{mirror_descent_derivations} identifies this exact policy-space update with NPG.
Thus, using our inverse-step-size notation $\eta=1/\alpha_k$,
the update is precisely $\I_{\mathrm{reg}}$.

\textbf{Trust Region Policy Optimization} \citep[][]{trpo,ppo} 
bounds the approximation error between the true RL objective and a local surrogate by constraining the policy update within a trust region defined by $KL(\pi'|\pi) \leq \delta$.
\citet{neu2017unified} analyze an exact, additively KL-regularized
TRPO surrogate in which the current policy is fully evaluated and
the current-policy state distribution is held fixed:
\begin{align}
    \pi_{k+1}
    \in
    \arg\max_{\pi'}
    \sum_s d^{\pi_k}(s)
    \left[
        \sum_a
        \pi'(a\mid s)A^{\pi_k}(s,a)
        -
        \eta
        D_{\mathrm{KL}}\!\left(
            \pi'(\cdot\mid s)\,\|\,\pi_k(\cdot\mid s)
        \right)
    \right].
\end{align}
Its exact tabular solution is
\begin{align}
    \pi_{k+1}(a\mid s)
    \propto
    \pi_k(a\mid s)
    \exp\left(
        A^{\pi_k}(s,a)/\eta
    \right).
\end{align}
which is exactly $\I_{\text{reg}}.$

\textbf{Regularized Policy Optimization} 
is motivated from the perspective of greedy updates often being unstable in practice \citep{neu2017unified}.
It models policy optimization using a weighted regularizer $\eta\mathcal{R}(\pi',\pi)$ (left), which is often chosen to be $\eta KL(\pi'|\pi)$, recovering \citep[right,][]{grill2020monte}:
\begin{align*}
     \argmax_{\pi'}
     \sum_{a \in \A}\pi'(a|s)Q^\pi(s,a) - \eta\mathcal{R}(\pi'(s),\pi(s)), 
     \,\,
     \I_{reg}(\pi,Q^\pi)(a|s) 
     = 
     \pi(a|s) \exp(Q^\pi(s,a)/\eta).
\end{align*}
In view of $\I_{opt}$, we can see that $\mathcal{R}(\pi',\pi)$ connects to a Bayesian prior $p(\pi'|\pi)$.
Similarly, $\eta(s,a)$ allows $\I_{opt}$ to regularize each state action \textit{differently}, as a function of the certainty in the advantage $A(s,a)$.
\section{Pseudo Code}
\label{app:psudocode}

\begin{algorithm}
    \caption{Generalized Policy Iteration}
    \label{alg:gpi}
    \begin{algorithmic}[1]
        \State For starting functions $q \in \Q, \, \pi \in \Pi $ greedification operator $\I$, $ k \geq 1 $ and $\delta > 0$
        \While{
                    $\displaystyle
                    \max_{s\in\mathcal S}
                    \left|
                        \sum_{a\in\mathcal A}\pi(a|s)q(s,a)
                        -
                        \max_b q(s,b)
                    \right|>\delta
                    \quad\textbf{or}\quad
                    \max_{(s,a)\in\mathcal S\times\mathcal A}
                    \left|
                        q(s,a)-\mathcal T^*q(s,a)
                    \right|>\delta
                    $
                }
            \State $ q(s,a) \gets (\T^\pi)^k q(s,a), \, \forall (s,a) \in \S \times \A $
            \label{alg:api:line:eval}
            \State $\pi(s) \gets \I(\pi, q)(s), \, \forall s \in \S$
            \label{line:pi}
        \EndWhile
    \end{algorithmic}
\end{algorithm}

The Bellman update $\T^\pi$ is defined as: $ \T^\pi q_{i}(s,a) = \mathbb{E}[R(s,a)] + \gamma \mathbb{E}_{s' \sim P}[\sum_{a' \in \A} \pi(a'|s') q_i(s',a')] $.
The notation $(\T^\pi)^k$ defines $k$ repeating Bellman updates of this form with the same policy $\pi$, repeatedly updating the value $q$.
$\T^*$ denotes the Bellman optimality operator, $  \T^* q_{i}(s,a) = \mathbb{E}[R(s,a)] + \gamma \mathbb{E}_{s' \sim P}[\max_{a' \in \A} q(s',a')] $.

\begin{algorithm}[H]
    \setlength{\baselineskip}{1.25\baselineskip}
    \caption{Simplified Model Based RL Training Loop with Modular Search}    
    \label{alg:outer_loop}
    \begin{algorithmic}[1]
        \Require Search algorithm (planner) $\mathcal{P}$, neural networks $\pi_{\theta_{1}}, v_{\phi_{1}}$, replay buffer $\mathcal{D}_{1} = \emptyset$, environment's dynamics model $\mathcal{M} = (P, R)$ and budget parameters $B$.
        \For{episode $n = 1$ to $N$}
            \State Sample starting state $s_1 \sim \rho $.
            \For{step $t = 1$ in the environment to termination or timeout}
                \State $\pi_{search}(s_t), V_{search}(s_t) \gets \mathcal{P}(\pi_{\theta_n}, v_{\phi_n}, \mathcal{M},B)(s_t) $.
                \State $a_t \sim \pi_{search}(s_t)$.
                \State $s_{t+1} \sim P(\cdot | s_t, a_t), \quad r_t \sim R(s_t, a_t)$.
                \State Append $(s_t, a_t, r_t, s_{t+1},\pi_{search}(s_t), V_{search}(s_t))$ to buffer $\mathcal{D}_{n}$.
            \EndFor
            \State Update policy params $\theta_{n+1}$ with SGD and CE loss on targets $\pi_{search}$ from $\mathcal{D}_{n}$.
            \State Update value params $\phi_{n+1}$ with SGD and MSE loss on TD-$\lambda$ targets using $V_{search}$ from $\mathcal{D}_{n}$.
            \State Set $\mathcal{D}_{n+1} \gets \mathcal{D}_{n}$.
        \EndFor
    \end{algorithmic}
\end{algorithm}
\section{Experimental and Implementation Details}
\label{app:details}

\subsection{Policy Iteration Experiments}
\label{app:details:dp}

The generalized Policy Iteration algorithm which accepts different policy improvement operators is described in Algorithm \ref{alg:gpi}.
In these experiments, we vary the improvement operator $\I$ used by the algorithm across different types of evaluation.
Two evaluation methods are used: \textit{exact evaluation} ($k \to \infty$) computed via matrix inversion and a single Bellman backup per iteration ($k = 1$), similar to Value Iteration and TD-learning.
Following both evaluation steps, the value is noised with independent zero-mean Gaussian noise with state-action-specific variances.
The experiments' results are presented in Figure \ref{fig:flagship} and Appendix \ref{app:additional_results}.

\textbf{Environments.}
The environments used in this experiment are infinite horizon grid-world mazes, using the ShinRL framework \citep{shinrl}.
To vary the transition and reward dynamics we generate 20 fixed randomly generated MDPs and run several trials per algorithm in each (Table \ref{tab:hyperparameters}).
Every trial starts from the uniform policy $\pi_0(a|s) = 1/|\A|.$

\textbf{Noise schedule.}
At each iteration every state--action pair draws a fresh variance:
\begin{align}
    c_i \sim \mathcal{U}\big(0,\, a_0\gamma_{\text{dec}}^{\,i}\big),
    \quad
    \sigma^2_{Q_i}(s,a) \sim \mathcal{U}\big(c_i,\, c_i + b\big),
    \quad
    \hat Q^{\pi_i}(s,a) \sim \N\big(Q^{\pi_i}(s,a),\, \sigma^2_{Q_i}(s,a)\big).
    \label{eq:dp_noise}
\end{align}
All constants are specified in Table \ref{tab:hyperparameters}.
A random half of the pairs instead keeps $\sigma^2_{Q_i} \sim \mathcal{U}(0,1)$ at every iteration to keep the variances of very mixed sizes, simulating the larger-error predictions from poorly or newly visited areas of the state space, as well as lower-error predictions of well visited, in RL. 
The floor $c_i$ decays as well, further simulating the reduction in evaluation error over the training process.

\textbf{Operators compared.}
All operators act on the noisy $\hat Q^{\pi_i}$ and its advantage
$\hat A^{\pi_i}(s,a) = \hat Q^{\pi_i}(s,a) - \sum_b \pi_i(b|s)\hat Q^{\pi_i}(s,b)$. Writing
$\sigma^2_{\hat A}$ for the advantage variance of Equation \ref{eq:advantage_likelihood}, the different operators are defined as follows:
\begin{align}
    &\I_{\text{opt}}(\pi_i, \hat Q^{\pi_i})(a|s) \propto \pi_i(a|s)\exp\Big(\tfrac{\tau^2}{(\tau^2 + \sigma^2_{\hat A}(s,a))\beta}\hat A^{\pi_i}(s,a)\Big),
    \\
    &\I_{\text{reg}}(\pi_i, \hat Q^{\pi_i})(a|s) \propto \pi_i(a|s)\exp\big(\hat A^{\pi_i}(s,a)/\eta\big),
    \\
    &\I_{\arg\max}(\pi_i, \hat Q^{\pi_i})(a|s) = \arg\max_a \hat Q^{\pi_i}(s,a),
    \\
    &\I_{\text{cpi}}(\pi_i, \hat Q^{\pi_i})(s) = (1-\alpha_i)\pi_i(s) + \alpha_i\,\argmax\nolimits_a \hat Q^{\pi_i}(s,a).
\end{align}
$\I_{\text{cpi}}$ \citep{approx_optimal_approx_rl} interpolates towards the greedy policy rather than jumping to it. 
For $\alpha_i$ we use \citet{safe_pi}'s more recent form, $\alpha_i \propto \frac{(1-\gamma)^2}{\gamma}\cdot \frac{|\S|^{-1}\sum_{s}B(s)}{\max_s B(s) - \min_{s'}B(s')}$ with $B(s) = \sum_a (\pi_{\text{greedy}}(a|s) - \pi_i(a|s))\hat Q^{\pi_i}(s,a)$.

In addition, we include pessimistic (LCB) variants of $\I_{\arg\max}$ and $\I_{\text{reg}}$ to investigate the differences between the effect of $\I_{\text{opt}}$'s uncertainty awareness and standard pessimism. 
They receive the same $\sigma^2_{Q_i}$ of Equation \ref{eq:dp_noise} but spend it as a value penalty,
$\hat Q^{\text{LCB}}_i(s,a) = \hat Q^{\pi_i}(s,a) - \lambda\sqrt{\sigma^2_{Q_i}(s,a)}$, before the unmodified operator is applied. 

\textbf{Score.}
Returns are min--max normalised per environment,
$V^{\pi_i}_{\text{norm}}(s_0) = \frac{V^{\pi_i}(s_0) - V^{\pi_0}(s_0)}{V^*(s_0) - V^{\pi_0}(s_0)}$,
where $V^{\pi_0}$ is the uniform policy and $V^*$ is obtained per seed by running a noise-free
$\argmax$ operator to convergence. Comparison and tuning use the normalised area under the
curve, $\mathrm{AUC} = T^{-1}\sum_{i=1}^{T} V^{\pi_i}_{\text{norm}}(s_0)$.
Curves aggregate all $N_{\text{envs}}\cdot N_{\text{trials}} = 2000$ scores per iteration. 
Every aggregate interval holds the environments fixed and treats only seeds as random,
so that $\mathrm{Var} = N_{\text{envs}}^{-2}\sum_e \sigma^2_e/n_e$ with $\sigma^2_e$ the
across-seed variance inside environment $e$. 
The hyperparameters for all operators were tuned with an extensive grid search, described in more detail in Appendix \ref{app:hyperparams}.

\subsection{Discrete-Action Model Based RL (GumbelAlphaZero on board games)}
\label{app:details:mcts}

GumbelAlphaZero uses GumbelMCTS to drive policy improvement and action selection in RL.
The GumbelMCTS algorithm is a modern variant of the MCTS algorithm which uses greedification operators directly within the search to drive policy improvement at the root.
In broad strokes, MCTS is an iterative algorithm that constructs a search tree and uses statistics to direct its search towards increasingly more promising directions.
In each iteration, the algorithm traverses the tree from the root using a \textit{search policy} $\pi_i = \I_{\text{reg}}(\pi_\theta, Q_i),$ where $ \pi_\theta $, often called the \textit{prior policy}, the policy trained by the RL algorithm, and $Q_i$ is the aggregated value observed during search.
When search arrives at an as-of-yet un-expanded action $a$ at a node $s_T$ (a \textit{leaf} in the tree), the node is expanded, evaluating $ s_{T+1} \sim P(\cdot | s_T,a), Q_i(s_T,a ) = r_T + \gamma V_\psi(s_{T+1}). $ 
The reward and transition models are both usually assumed to be deterministic.
After the expansion step, the new information $r_T + \gamma V_\psi(s_{T+1})$ observed in the expansion is backpropagated throughout the entire search trajectory all the way back to the root, the iteration has completed, and the next iteration starts from the top.
The number of iterations, which is also the number of expansions, $M$, is the common compute budget considered for MCTS, as in many cases the DNN forward-pass-based expansion is the cost's bottleneck.

The GumbelAlphaZero (GAZ) algorithm is a natural test-bed for $\I_{\text{opt}}$: its visit counts provide a natural state-action evaluation-variance estimator under the classical sample-mean interpretation of MCTS and its DNN-based evaluations test $\I_{\text{opt}}$ in fully the modern deep-RL setting. 
Finally, it already has a clear interpretation as using iterative greedification within the tree for policy improvement outside the tree, with $\I_{\text{reg}}.$

Evaluation in board game, a standard, hard, benchmark domain for RL and MCTS, is very expensive: the agents take relatively large number of steps before they have well functioning policies.
To achieve high-quality training curves, the evaluation requires comparing every checkpoint to every other checkpoint across a large number of games, where each agent uses MCTS variants with the checkpoint's DNN.
Due to the high cost of this evaluation, we run two different experiments, answering two faces of the same question - \textit{how does GumbelAlphaZero perform when $\I_{\text{reg}}$ is replaced with $\I_{\text{opt}}$?} 

In the first experiment we compare 4 agents: with and without a $Q_\phi$ critic, and with and without $\I_{\text{opt}}$ in place of $\I_{\text{reg}}.$
This experiment allows us to evaluate \textit{which operator, in and out of the presence of the additional critic, produces better learning?}
This is the experiment presented in Figure \ref{fig:flagship}.
In the second experiment we compare 6 agents - the same 4 with additionally a vanilla-AlphaZero \citep{AlphaZero} PUCT-based \citep{AlphaGo,puct} agent, with and without a $Q_\phi$ critic.
In this experiment, we use \textit{the same} DNNs ($\pi_\theta, V_\psi, Q_\phi$) for all agents ($Q_\phi$ when relevant).
Rather than learning curves, this produces the playing strength of MCTS with each operator.
This answers the question \textit{what is the direct contribution of each operator to action selection - playing strength - when all else is equal?} 
We also repeat this experiment across different planning budgets $M$ and in two environments: chess and 9x9 Go, allowing us to answer this question much more comprehensively.

The experiments' results are presented in Figure \ref{fig:flagship} in Section \ref{sec:experiments} and Figures \ref{fig:mcts_curve_seeds} and \ref{fig:mcts_M_grid} in Appendix \ref{app:additional_results}.

\subsubsection{The operator}
\label{app:details:mcts:operator}

GumbelMCTS \citep{gumbel_mz} uses $\I_{reg}$ for three purposes: (I) Policy improvement at the root of the tree, once search has concluded.
(II) To drive the search at the root, using Sequential Halving \citep{sh}.
(III) To drive the search within the tree, deciding which action to select at each state (each node) in the search tree.
All three use the \textit{completed $Q$-values} mechanism and a single node-level
coefficient:
\begin{align}
    \I_{reg}(\pi_\theta, Q_i)(a|s) \propto \pi_\theta(a|s)\exp\Big(\frac{1}{\eta(s)}Q_i(s,a)\Big),
    \,\,
    \frac{1}{\eta(s)} = c_{scale}\big(c_{visit} + \max_{b \in \A} N(s,b)\big),
    \label{eq:gmz_operator}
\end{align}
where $N(s,a)$ is the number of simulations that have passed through $(s,a)$ and $Q_i$
is the value estimated by the search tree at search iteration $i$. This coefficient is a heuristic for
the variance reduction produced by repeated evaluation: under a fixed rollout distribution and
independent sampling, the variance of a sample mean decreases as $1/N(s,a)$, so the inverse visit count is proportional to the variance. 

We replace every use of $\I_{\text{reg}}$ with $\I_{\text{opt}}$, setting $\sigma_Q^2(s,a)\propto1/N(s,a)$, the corresponding classical evaluation variance estimator, 
obtain $\sigma^2_A(s,a)$ from Equation \ref{eq:advantage_likelihood}, and
$\eta(s,a)$ from Equation \ref{eq:iopt}:
\begin{align}
    \I_{\text{opt}}(\pi_\theta, Q_i)(a|s) &\propto \pi_\theta(a|s)\exp\Big(\frac{1}{\eta(s,a)}A_i(s,a)\Big).
\end{align}
To avoid division by zero, actions the search has not visited take $N(s,a)=1$, and all other actions add additionally the number of evaluations through the node.
The value of actions the search has not visited is completed by the network's per-action $Q_\phi(s,a)$ where a $Q$ head is available and otherwise by the completed-Q value mechanism used by \citet{gumbel_mz}.

\subsubsection{Training the Actor and the Critics}
\label{app:details:mcts:qhead}

The actor $\pi_\theta$ is trained in the standard GumbelAlphaZero manner, with a cross entropy loss with the target $\pi_{\text{search}}(s) = \I(\pi_\theta, Q_M),$ the improved policy at the root of the search tree at the last iteration $M$ \citep{gumbel_mz}.
Baseline GumbelAlphaZero uses $\I_{\text{reg}}$, while our version uses $\I_{\text{opt}}$.

The state-value critic $V_\psi(s)$ is trained with MSE loss with the modern model-based target $V_{\text{search}}(s) = \sum_{a \in B}\bar \pi_{\text{search}}(a|s)Q_M(s,a) $ \citep{oren2025epistemic,trtpi,tsmcts,pmcts}.
The set $B$ denotes the set of all \textit{searched} actions at the root, and the $\bar \pi_{\text{search}}$ the improved policy at the root, renormalized only over the searched actions.

MCTS does not traditionally train a Q-head, although it has been tried before \citep{gao2018three}.
However, since the popular, canonical approaches, AlphaZero \citep{AlphaZero}, MuZero \citep{muzero} and GumbelAlpha/MuZero \citep{gumbel_mz} from DeepMind do not use a Q head, there is no one accepted canonical approach.
Building on recent ideas and developments from MCTS and AlphaZero/MuZero literature \citep{oren2025epistemic,tsmcts} we train the Q-network as follows.
The search at the root splits the action set in three, each receiving a different target:
\begin{enumerate}
    \item \textbf{Searched actions} ($N(s,a) > 0$) take the tree's own backed-up value $Q_i$ at the end of search.
    \item \textbf{Bootstrapped actions:} $L$ of the unsearched legal actions are sampled from the improved policy at the end of search, without replacement and given the one-step TD-style target $T(s,a) = r(s,a) + \gamma\, V_\psi(s')$. 
    \item \textbf{All other actions} receive no target and contribute zero gradient.
\end{enumerate}
The error vs. the target is measured with MSE.

Each searched action is weighted by $\sqrt{1 + N(s,a)}$, i.e.\ by
inverse \emph{standard error} rather than inverse variance. Precision weighting $(1+N)$ is
maximum-likelihood only for unbiased targets, and these are not: an action is most-visited
partly because its early estimates looked good, so its backup carries a winner's curse that
does not shrink like $1/N$. Bootstrapped actions get unit weight; their noise is that of
$V_\psi(s')$ and is not on the $\sqrt{1+N}$ scale.

The weighted regression constrains only actions that received a target, so we add one scalar constraint per state: the improved policy at the root $\pi_{\text{search}}$-weighted mean of $Q_\phi$ must match the same mean over the tree's completed values, with the target treated as a constant, providing some learning signal to all actions. 

The total loss is
$L = L_\pi + L_v + \lambda_q L_q$, with $L_\pi$ the cross-entropy against the improved policy
over legal actions, $L_v$ the squared error on the value head, and $L_q$ the sum of the
weighted regression and the calibration term. 
Hyperparameters are specified in 
Table \ref{tab:hp_mcts}.

\subsubsection{Evaluation across training}
\label{app:details:mcts:curves}

\paragraph{Training.}
We use \texttt{pgx} \citep{pgx} implementation of $9\times9$ Go and chess.
All agents follow the training scheme of AlphaZero \citep[][Algorithm \ref{alg:outer_loop}]{AlphaZero,gumbel_mz}.
Actions are sampled randomly from the acting policy for the first 8 moves and then acting is deterministic.
The acting policy follows the conventions of AlphaZero and GumbelAlphaZero (visits at the root, and / or the action chosen by Sequential Halving in the GumbelAlphaZero case, see \citep{gumbel_mz} and \citep{sh}).
All agents share every hyperparameter except the
operator and the $Q$ head, detailed in Table \ref{tab:hp_mcts}. 

\paragraph{Evaluation.}
Evaluation with Bayes elo is conducted by playing every DNN checkpoint and search pair against every other.
We train 5 seeds per agent and evaluate all of them in one tournament, so every
checkpoint of every seed plays every other and the whole set is fitted on a single Elo scale.
The results are presented in Figure \ref{fig:mcts_curve_seeds}. 
We report the mean over seeds and the minimum and maximum elo observed over seeds. 
BayesElo's own interval is reported per seed.

\subsubsection{Evaluation at a fixed network}
\label{app:details:mcts:tournament}

\paragraph{Setup.}
In the equalized-DNNs experiment, all agents use the same checkpoint trained with GAZ + $Q_\phi,$ to make sure no advantage is given to $\I_{\text{opt}}$ from training its own networks.

\paragraph{Evaluation.}
Agents play a round robin from a fixed opening book, with every pairing played in both colours
so that a first-move advantage term is not identifiable. Elo is fitted with BayesElo holding the \textit{colour-advantage} at zero.

\subsection{Classical Control Model-Free RL (SAC on DMC)}
\label{app:details:sac}

\paragraph{The Operator.}
SAC can be interpreted as policy iteration with the regularized policy-improvement operator
$
\pi_{i+1}(a|s)
=
\mathcal I_{\mathrm{reg}}
\left(
\mu,\hat Q_{\mathrm{soft}}^{\pi_i}
\right)(a|s).
$
In a tabular or discrete-action setting, this improved policy can be computed explicitly over the action space. With a parameterized policy \(\pi_\theta\), the improved policy can be projected onto the policy class directly using a KL loss with respect to the target $\pi_{i+1}$. 
In continuous-action settings, explicit construction and projection of the improved policy $\pi_{i+1}(a|s)$ can also be performed using sampled actions, as in MPO \citep{mpo}.

This is not, however, the approach taken by commonly used continuous-action implementations of SAC. In popular implementations (\cite{huang2022cleanrl}'s implementation for example which we use in the experiments), the actor is optimized directly using the reparameterized loss
\begin{align}
L(\theta)
&=
\eta \log \pi_\theta(a_\theta|s)
-
A_{\mathrm{soft}}(s,a_\theta),
\label{eq:sac_loss}
\end{align}
where \(a_\theta=f_\theta(\epsilon;s)\), with \(\epsilon\) denoting sampled noise. The reparameterization trick allows gradients to propagate through the sampled action to the policy parameters \citep{sac1}. Here, using \(A_{\mathrm{soft}}\) instead of \(Q_{\mathrm{soft}}\) is equivalent for the actor optimization, since
$
    A_{\mathrm{soft}}(s,a)=Q_{\mathrm{soft}}(s,a)-V_{\mathrm{soft}}(s),
$
and \(V_{\mathrm{soft}}(s)\) is independent of the action and therefore does not affect the gradient with respect to \(\theta\).

Thus, rather than explicitly constructing and projecting onto the improved policy, SAC directly optimizes the regularized objective whose unconstrained policy-space optimum is \(\mathcal I_{\mathrm{reg}}\).

Our objective is to evaluate whether \(\mathcal I_{\mathrm{opt}}\) can directly improve existing RL algorithms while retaining their practical design choices and tuning procedures. We therefore preserve the direct gradient-based actor optimization used by SAC. To incorporate \(\mathcal I_{\mathrm{opt}}\) while retaining SAC's self-tuned temperature parameter, we set $ \eta = \frac{\beta}{\tau^2}$ and incorporate the uncertainty-dependent weighting prescribed by \(\mathcal I_{\mathrm{opt}}\). Specifically, we use
\begin{align}
L_{\mathcal I_{\mathrm{opt}}}(\theta)
=
\eta \log \pi_\theta(a_\theta|s)
-
\frac{
A_{\mathrm{soft}}(s,a_\theta)
}{
\tau^2+\sigma^2_{A_{\mathrm{soft}}}(s,a_\theta)
}.
\label{eq:iopt_sac_loss}
\end{align}
This leaves \(\tau^2\) as the only additional hyperparameter introduced by \(\mathcal I_{\mathrm{opt}}\).

For the uncertainty estimate, rather than combining the variances of the two value estimates as in Equation \ref{eq:advantage_likelihood}, which neglects their covariance, we compute the advantage separately using each critic and estimate its variance from their disagreement:
\begin{align}
\sigma^2_{A_{\mathrm{soft}}}(s,a_\theta)
=
\operatorname{Var}_i
\left[
Q_{\phi_i}(s,a_\theta)
-
\mathbb{E}_{b \sim \pi_\theta(\cdot|s)}
\big[
Q_{\phi_i}(s,b)
\big]
\right].
\label{eq:sac_sigma}
\end{align}

The inner expectation is each critic's own state value $V_{\phi_i}(s)$, estimated by Monte Carlo with $32$ actions drawn from $\pi_\theta(\cdot|s)$ at that state. Each critic therefore forms its own advantage of the sampled action $a_\theta$ over the policy's value, and the variance across critics of those advantage estimates is the uncertainty estimate. Centring each critic on its own value removes that critic's state-dependent bias, so what remains is disagreement about the action rather than about the state; with the two critics used here the estimator reduces to $\tfrac{1}{2}\left(A_1 - A_2\right)^2$. This requires no additional networks, at the cost of $32$ further action samples and critic evaluations per state.

The implementation weights $Q_{\mathrm{soft}}(s,a_\theta)$ rather than $A_{\mathrm{soft}}(s,a_\theta)$ in Equation \ref{eq:iopt_sac_loss}; since the detached weight is treated as constant during differentiation, the two give identical actor gradients, by the argument given above for Equation \ref{eq:sac_loss}.

In this implementation, gradients are stopped through the uncertainty estimate \(\sigma^2_{A_{\mathrm{soft}}}(s,a_\theta)\). Consequently, uncertainty determines the weighting assigned to each sampled state-action gradient, but does not contribute directly to the gradient through its dependence on the action. In particular, the resulting update retains the term:
\begin{align}
\frac{1}{
\tau^2+\sigma^2_{A_{\mathrm{soft}}}(s,a_\theta)
}
\nabla_\theta A_{\mathrm{soft}}(s,a_\theta),
\end{align}
while omitting the gradient of the uncertainty-dependent weight itself. 
Because gradients are stopped through the uncertainty estimate,
uncertainty introduces no additional action-space gradient through
$\nabla_a \sigma^2_{A_{\mathrm{soft}}}$. For each sampled
state-action pair, it only positively rescales the ordinary advantage
gradient in Equation~\ref{eq:iopt_sac_loss}. 
Thus, under direct optimization of a single action at a fixed state, it does not
change the stationary points or maximizers of the $Q$ term. 
As demonstrated in our experiments, this approximation is sufficient to improve performance while retaining the original SAC optimization procedure.

\paragraph{Setup.}
We use ten popular DeepMind Control Suite tasks \citep{dm_control} and train the agents up to 1$M$ steps. 
The environments are:
\texttt{cheetah-run}, \texttt{humanoid-run}, \texttt{quadruped-run}, \texttt{fish-swim}, \texttt{humanoid-stand}, \texttt{quadruped-walk}, \texttt{hopper-hop}, \texttt{humanoid-walk}, \texttt{walker-walk} and \texttt{hopper-stand}.

\paragraph{Metric and aggregation.}
Returns are min--max normalised per task over all seeds and then aggregated per agent across domains.
We report the performance after $1M$ steps.

Hyperparameters are specified in Appendix \ref{app:hyperparams}, Table \ref{tab:hp_sac}.

\subsection{Offline RL (ReBRAC on D4RL)}
\label{app:details:offline}
\paragraph{The Operator.}
ReBRAC \citep{tarasov2023revisiting} trains a deterministic actor to maximize
$Q_\phi(s,\pi_\theta(s))$ under a behavior-cloning penalty.
This penalty serves as the behavior-regularization term corresponding to the
divergence from the prior in Equation~\ref{eq:ce_decision_objective}.
A deterministic actor does not provide a tractable Boltzmann policy target to
project onto, as in the SAC case (Appendix~\ref{app:details:sac}).
We therefore apply the coefficient prescribed by Equation~\ref{eq:iopt} directly to ReBRAC's policy-improvement term:
\begin{align}
L_{\mathcal{I}_{\mathrm{opt}}}(\theta)
=
-\frac{\lambda}{\beta}
\frac{\tau^2}
{\tau^2 + \sigma_{\hat A}^2(s,a)}
Q_\phi\bigl(s,\pi_\theta(s)\bigr)
+
\beta_1
\left\|
\pi_\theta(s)-a
\right\|_2^2.
\label{eq:rebrac_topi}
\end{align}

The advantage with respect to a policy $\pi$ and its value functions
$Q^\pi$ and $V^\pi$ is defined as
$A^\pi(s,a)=Q^\pi(s,a)-V^\pi(s)$.
In ReBRAC, however, the policy being evaluated and the policy serving as the
behavior-regularization reference differ.
The critic $Q_\phi \approx Q^{\pi_\theta}$ evaluates the learned policy
$\pi_\theta$, while the behavior policy $\pi_b$, represented by the actions in
the offline dataset, serves as the reference toward which the actor is
regularized.

We therefore use:
\begin{align}
\hat A(s,a)
=
Q_\phi\bigl(s,\pi_\theta(s)\bigr)
-
Q_\phi(s,a),
\qquad
Q_\phi(s,\cdot)
=
\min_i Q_{\phi_i}(s,\cdot),
\label{eq:rebrac_advantage}
\end{align}
as the improvement quantity for constructing the uncertainty estimate.
This quantity measures the predicted improvement of the learned action relative to ReBRAC's behavior reference, with both actions evaluated by the learned critic.
Since the dataset action $a$ does not depend on the actor parameters, we have:
\begin{align}
\nabla_\theta \hat A(s,a)
=
\nabla_\theta Q_\phi\bigl(s,\pi_\theta(s)\bigr),
\end{align}
and hence $\hat A$ induces exactly the same actor gradient as ReBRAC's original
$Q$-maximization term when their coefficients are fixed.
Estimating uncertainty in this difference therefore measures critic disagreement
about the improvement of the learned action relative to the behavior action,
rather than disagreement about their absolute values.

To estimate this uncertainty, rather than combining the variances of the two
value estimates as in Equation~\ref{eq:advantage_likelihood}, which neglects their
covariance, we compute the improvement separately using each critic and estimate
its variance from their disagreement:
\begin{align}
\sigma_{\hat A}^2(s,a)
=
\operatorname{Var}_i
\left[
Q_{\phi_i}\bigl(s,\pi_\theta(s)\bigr)
-
Q_{\phi_i}(s,a)
\right].
\label{eq:rebrac_sigma}
\end{align}
That is, each critic evaluates both the current policy action $\pi_\theta(s)$
and the dataset action $a$, and the variance across critics of the resulting
improvement estimates is used as the uncertainty estimate.
Evaluating both actions with the same critic also allows errors shared between
the two evaluations to cancel in their difference, rather than treating the two
value estimates as independent sources of uncertainty.
The resulting quantity therefore captures critic-ensemble disagreement about
the estimated improvement of the current policy action over the action observed
in the data.
This requires no additional networks: with two critics, each critic evaluates
both action inputs, in addition to the single actor evaluation required to
obtain $\pi_\theta(s)$.

Writing
\begin{align}
w(s,a)
=
\frac{\tau^2}
{\tau^2 + \sigma_{\hat A}^2(s,a)}
\end{align}
for the uncertainty-dependent coefficient, which is treated as constant with
respect to the actor gradient, we apply $w(s,a)$ directly to ReBRAC's original
$Q$-maximization term.
ReBRAC's $Q$-scale normalizer $\lambda$ is correspondingly computed from the
weighted term as it actually enters the loss:
\begin{align}
\lambda
=
\left(
\mathbb{E}_{(s,a)\sim\mathcal{D}}
\left[
\left|
w(s,a)
Q_\phi\bigl(s,\pi_\theta(s)\bigr)
\right|
\right]
\right)^{-1}.
\label{eq:rebrac_lambda}
\end{align}
This replaces ReBRAC's original normalizer
\begin{align}
\left(
\mathbb{E}_{(s,a)\sim\mathcal{D}}
\left[
\left|
Q_\phi\bigl(s,\pi_\theta(s)\bigr)
\right|
\right]
\right)^{-1},
\end{align}
so that ReBRAC's $Q$-maximization term remains scale-normalized after
introducing the uncertainty weighting.

Although $\hat A(s,a)$ and $Q_\phi(s,\pi_\theta(s))$ induce identical actor
gradients when their coefficients are fixed, their numerical values generally
differ and they are therefore not interchangeable inside the scale normalizer.
Computing $\lambda$ from
$\mathbb{E}[|w(s,a)\hat A(s,a)|]$ instead would generally produce a different
$\lambda$, thereby changing the relative scale of the policy-improvement
gradient and the behavior-cloning gradient.
Equation~\ref{eq:rebrac_lambda} instead preserves ReBRAC's normalization
convention by forming $\lambda$ from the weighted $Q$-maximization term that
actually enters Equation~\ref{eq:rebrac_topi}.
In particular, when $w(s,a)=1$, Equation~\ref{eq:rebrac_lambda} recovers
ReBRAC's original normalizer.

The critic, both behavior-cloning coefficients, and all other components of ReBRAC are left unchanged.

\paragraph{Setup.}
Four D4RL \citep{d4rl} Gym-MuJoCo medium-replay datasets: \textit{halfcheetah}, \textit{hopper}, \textit{walker2d} and \textit{ant}.
Each run trains for $1$M gradient steps. 

\paragraph{Metric.}
The reported metric is the final policy evaluated over $56$ episodes. 
The offline bar of Figure \ref{fig:flagship} is min--max normalised per environment over the runs of both arms and then averaged over the four environments, the normalisation the SAC
bar of the same figure uses. 
Table \ref{tab:offline_per_env} instead reports raw D4RL points
($0$ = random, $100$ = expert).
\section{Additional Results}
\label{app:additional_results}
In this section we provide additional results across the different experiment settings presented in Figure \ref{fig:flagship} in Section \ref{sec:experiments}.
These include the individual results aggregated in Figure \ref{fig:flagship},
the behavior of the hyperparameters introduced with $\I_{\text{opt}}$ and additional experiments to answer specific research questions and further investigate the contribution of $\I_{\text{opt}}$ to performance in different domains and algorithms.

\subsection{Dynamic Programming additional results}
\label{app:additional_results_dp}

\subsubsection{Generalized Policy Iteration with $k \to \infty$}
We include a set of additional experiments and results in the Policy Iteration setup presented in Figure \ref{fig:flagship}.
In Figure \ref{fig:results_dp} left we provide the full learning curves (improvement of the value of the starting state across iterations), summarized in Figure \ref{fig:flagship}.

In Figure \ref{fig:results_dp} center we show the sample efficiency, measured as \textit{area under the curve} (AUC), across increasing starting values of the noise schedule.
We see that the dominance of $\I_{\text{opt}}$ grows significantly as the error grows.

In Figure \ref{fig:results_dp} right we investigate the dependence of the operator on the hyperparameters $\beta, \tau, $ which show a smooth ridge behavior, concentrating around smaller $\beta$ and $\tau$.
This also supports the operator's ability to exploit $\sigma_A$, as the effect of $\sigma_A$ reduces with the size of $\tau$, and the larger $\tau$ the worse the performance.

\begin{figure}[H]
    \centering
    \includegraphics[width=1\linewidth]{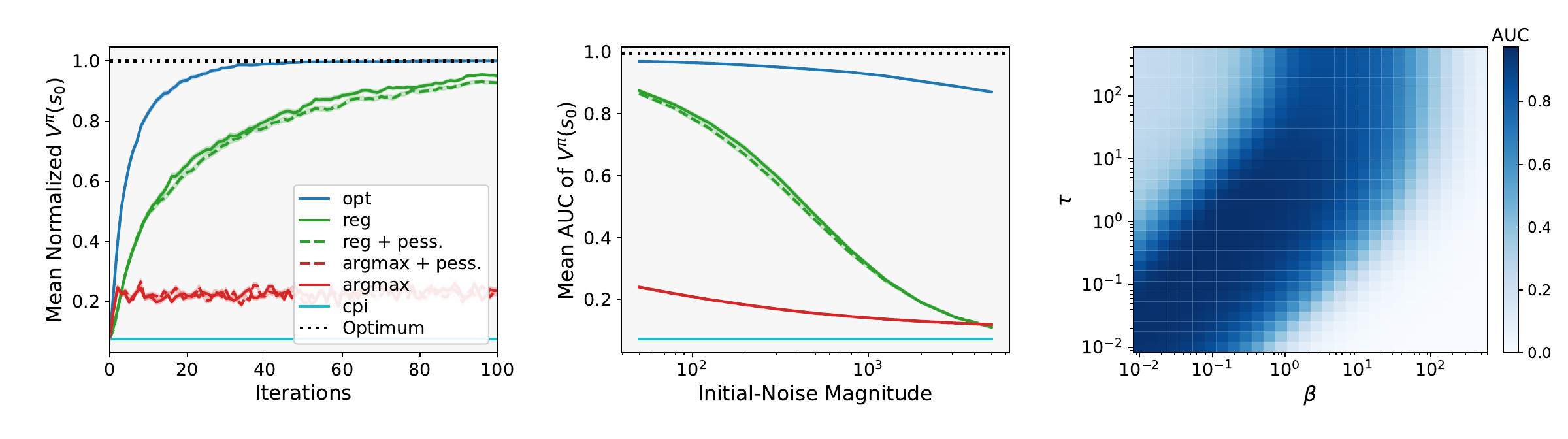}
    \vspace{-0.7cm}
    \caption{
    Policy Iteration experiments with different policy improvement operators and uncertain evaluation, mean and 95\% Gaussian CI. 
    \textbf{Left:} Example learning curves.
    \textbf{Center:} Sample-efficiency vs. starting noise magnitude, measured as Area Under the Curve (AUC) of learning curves until 100 iterations (higher is better).
    \textbf{Right:} AUC vs. hyperparameter combination. 
    The performance landscape exhibits a single broad ridge of high performance, with performance degrading smoothly away from the ridge.
    }
    \label{fig:results_dp}
\end{figure}

\subsubsection{Generalized Policy Iteration with $k=1$}
\label{app:additional_results_dp_td}

In the majority of practical RL settings, we would not expect the evaluation to be truly unbiased \citep{hullermeier2021aleatoric}.
To investigate the behavior of $\I_{\text{opt}}$ in a setting that violates this idealized unbiased-evaluation assumption while remaining controlled (decoupled from additional effects such as exploration, whether there are true Gaussian components to the error, etc.), we repeat the Policy Iteration experiment presented in Figure~2 exactly but with $k=1$. 
That is, rather than exact evaluation followed by Gaussian noise, we use a single Bellman backup followed by the same noise, introducing approximation and bootstrapping error relative to $Q^{\pi_i}$ in the manner of Value Iteration and TD learning.
The results are presented in Figure \ref{fig:results_dp_td}.

In Figure \ref{fig:results_dp_td} left we see a single set of learning curves, where the conclusion holds: $\I_{\text{opt}}$ remains the best performer, albeit with a smaller advantage.

In Figure \ref{fig:results_dp_td} center we see the behavior with growing error schedules.
Although $\I_{\text{opt}}$ remains the best performer, the advantage, which previously increased with error, now remains much more constant.

In Figure \ref{fig:results_dp_td} right we see that the behavior of the hyperparameters remains similar in that it is roughly a ridge, however the values of the optimal $\tau$ and $\beta$ are now much larger.

\begin{figure}[H]
    \centering
    \includegraphics[width=1\linewidth]{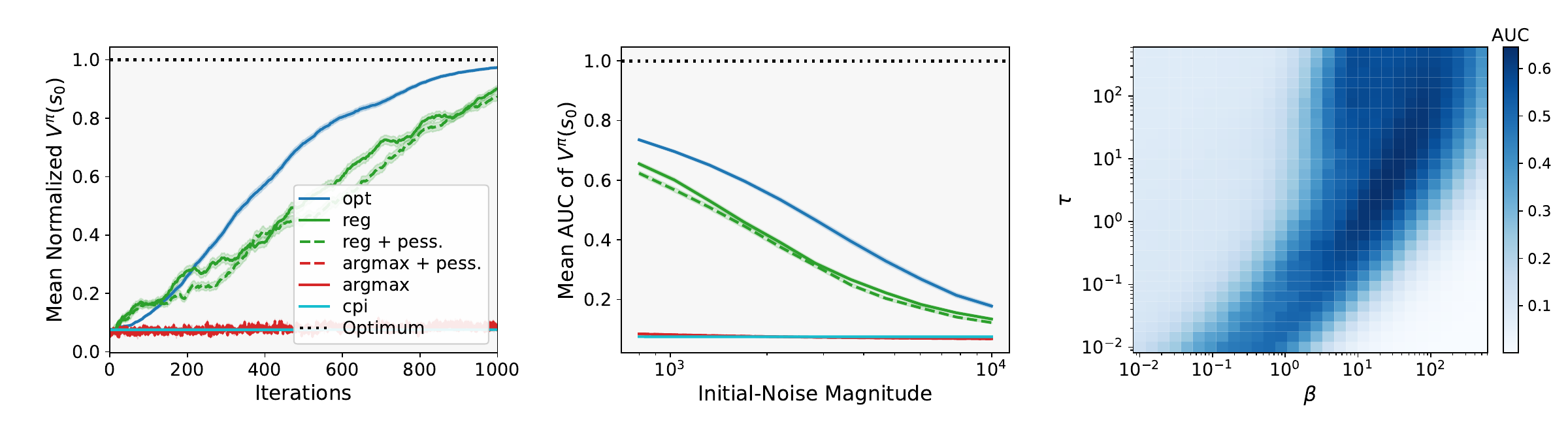}
    \vspace{-0.7cm}
    \caption{
    Policy Iteration with a 1-step Bellman update ($k=1$), otherwise identical to
    Figure \ref{fig:results_dp}.
    }
    \label{fig:results_dp_td}
\end{figure}

\subsection{Discrete-action model-based, additional results}
\label{app:additional_mcts_results}
The contribution of $\I_{\text{opt}}$ to search-based policy improvement can be evaluated in different ways.
In Figure \ref{fig:flagship} we include the comprehensive evaluation, comparing the strength of the agent that learns with each search variant, and then evaluates with it.
We include checkpoints during training of all seeds of all agents in a single tournament below.
Additionally, the contribution of $\I_{\text{opt}}$ to search based policy improvement can be evaluated more directly by comparing the playing strength of different MCTS variants using a \textit{single} set of DNNs $\pi_\theta, V_\psi, Q_\phi$ in all variants.
We include this comparison second.
Since this comparison is much less compute intensive, we are able to expand it along an additional MCTS variant, PUCT, which is used by vanilla AlphaZero and in an additional environment, chess, and along different search-budgets $M$.
Additionally, we include an investigation of the behavior of the error in MCTS with DNNs, where we find that indeed the \(Q_\phi\) residuals have higher variance than the \(V_\psi\) residuals.

\subsubsection{Elo-based learning curves across agents and seeds}
\label{app:additional_results_seeds}
We include the complete learning curves across seeds in Figure \ref{fig:mcts_curve_seeds}.
All training checkpoints of all agents play against all other checkpoints and each other, forming true Bayes-Elo curves.

\begin{figure}[t]
    \centering
    \includegraphics[width=0.6\linewidth]{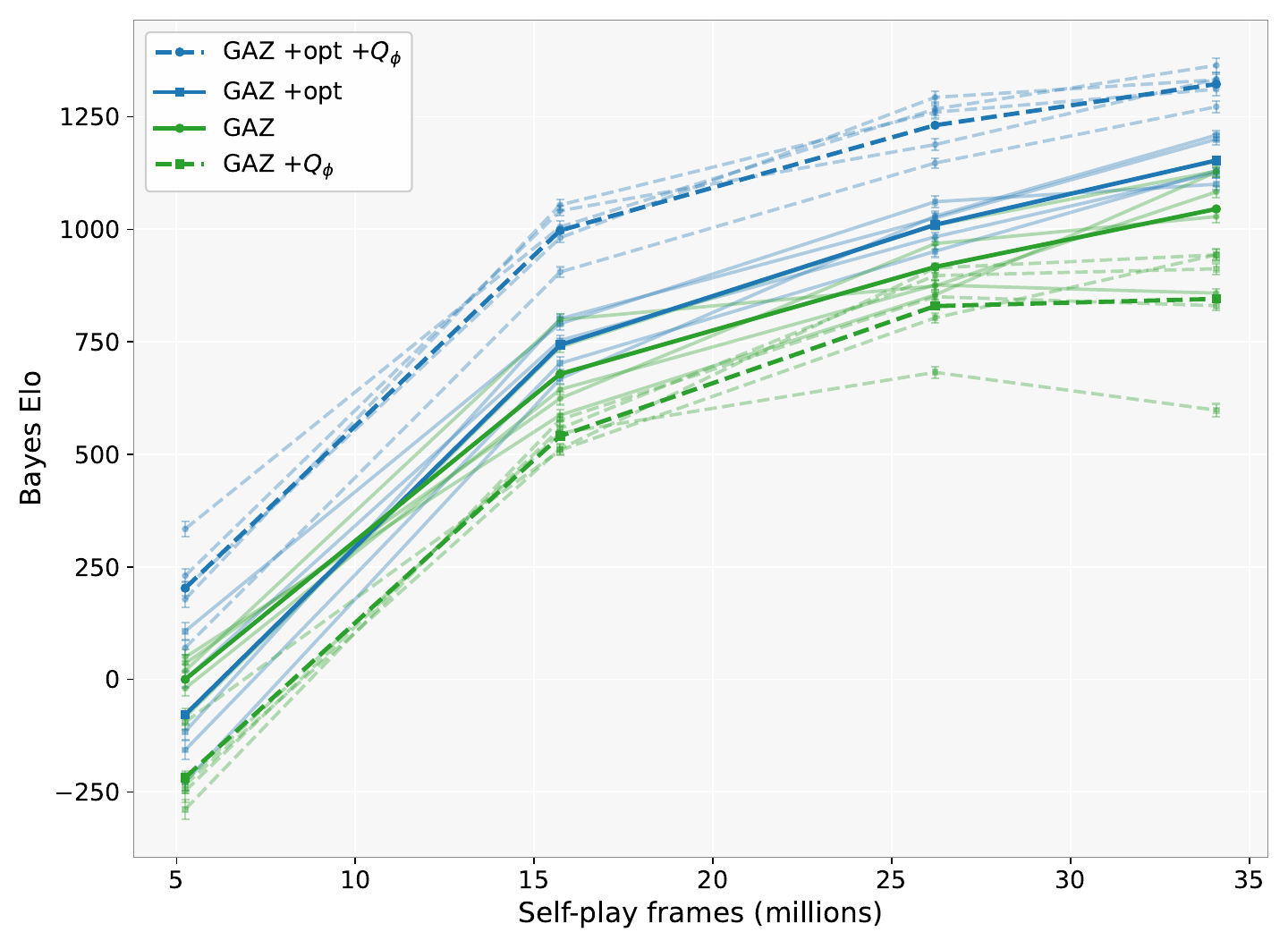}
    \vspace{-0.5cm}
    \caption{
    Playing strength in training on $9\times9$ Go with $M=64$. 
    Every checkpoint of every agent and every seed plays every other resulting in a single set of Elo scores. 
    Bold curves are the mean over 5 training seeds per agent. 
    Faint curves are the seeds themselves, with $95\%$ BayesElo CI.
    }
    \label{fig:mcts_curve_seeds}
\end{figure}

\subsubsection{
Playing strength across budgets \texorpdfstring{$M$}{M} and games 9x9 Go and chess
}
\label{app:additional_results_msweep}

We evaluate the effect of different operators used for search in MCTS for playing strength directly.
We use the same set of DNNs for all agents to isolate the effect of the search operator and avoid confounding it with differences in network training. 
The DNNs were trained with GumbelAlphaZero rather than with $\I_{\text{opt}}$.
The results are presented in Figure \ref{fig:mcts_M_grid}.

Across all $M$ and both games, $\I_{\text{opt}} $ paired with $Q_\phi$ outperforms all other variants.
Across most values of $M$, $\I_{\text{opt}} $ matches or outperforms $\I_{\text{reg}}. $ 
However, this does not hold for small values of $M = 8, 24,$ demonstrating a certain apparent dependence of $\I_{\text{opt}}. $
We note that only one set of hyperparameters was used for all agents, so we would not extract the conclusion here that one agent is better than the others in a certain $M,$ but rather that the target budget should be taken into account when tuning the hyper parameters of each agent (see the next subsection in this Appendix, discussing tuning the parameters of $\I_{\text{opt}} $).

PUCT is the worst performer in almost all domains, but benefits much more from $Q_\phi$ than $\I_{\text{reg}}, $ which could be expected due to its explicit dependence on a $Q_\phi$ prediction per state and action, and lack of the completed-Q-values mechanism developed by \cite{gumbel_mz} for GumbelMCTS.
Nonetheless, PUCT + $Q_\phi$ never competed with $\I_{\text{opt}} $ in playing strength in these experiments.

\begin{figure}[H]
    \centering
    \includegraphics[width=1\linewidth]{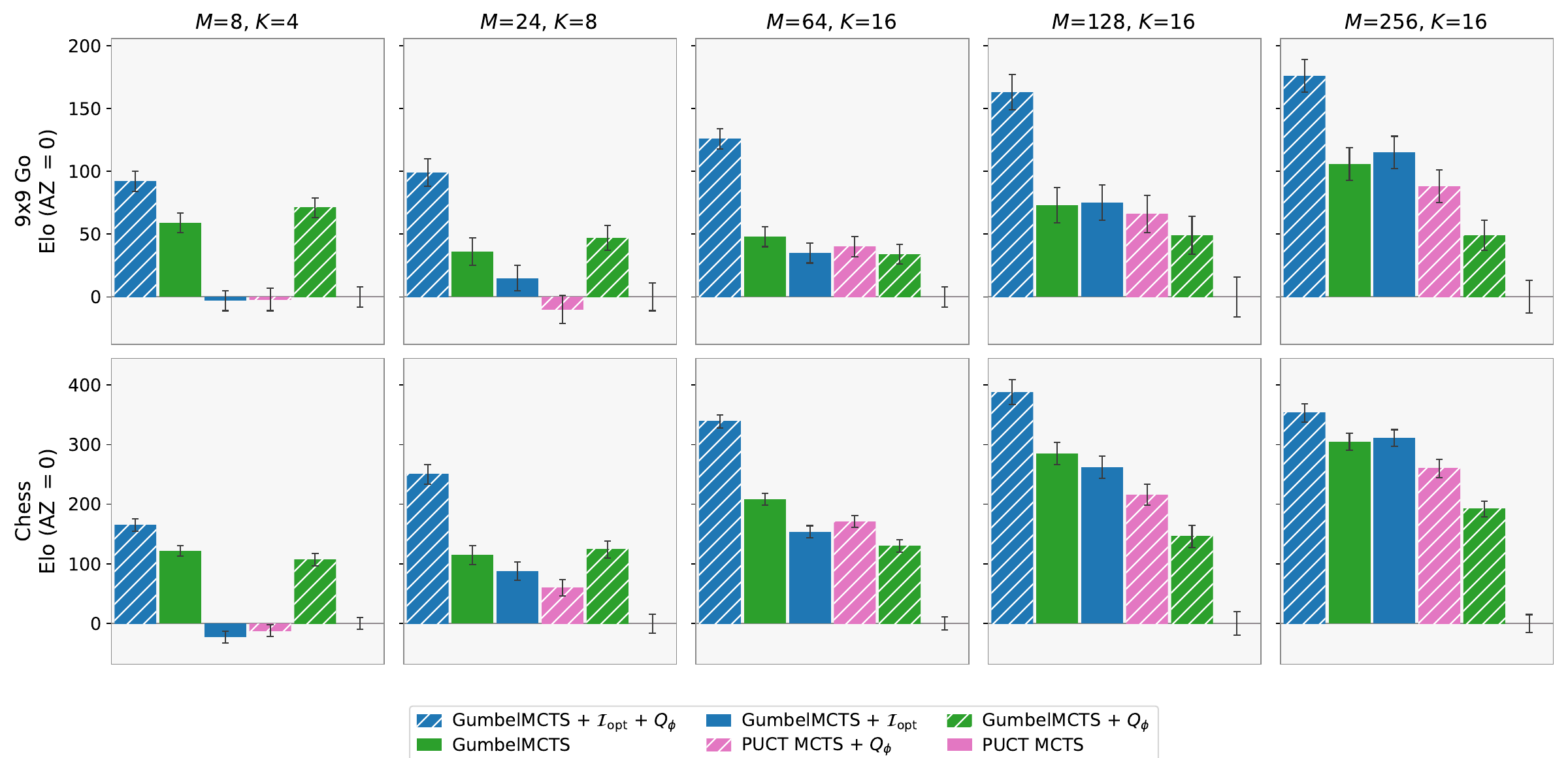}
    \vspace{-0.7cm}
    \caption{
    Playing strength per agent with the same set of DNNs.
    Bayes elo and 95\% CIs.
    }
    \label{fig:mcts_M_grid}
\end{figure}

\subsection{
Investigation of error in MCTS with DNNs
}
\label{app:additional_results_qv}
\label{app:additional_results_gauss}

In Section \ref{sec:experiments} we hypothesize that $\I_{\text{opt}}$ can be expected to benefit significantly from a $Q_\phi$ critic over the completed-Q-values mechanism used by baseline GumbelAlphaZero (GAZ), for two complementary reasons.
On the one hand, the increased information provided by a unique $Q^{\pi_\theta}$ prediction per state and action, compared to a single prediction of $V^{\pi_\theta}(s)$ replacing the prediction of the value of every action. 
On the other hand, the presence of larger prediction variance in $Q_\phi$ for which $\I_{\text{opt}}$ is designed to account and $\I_{\text{reg}}$ is not.
In this section, we investigate the behavior of the error empirically and whether this hypothesis is supported.

In addition, we investigate to what extent the DNN evaluation residuals agree with the
heteroscedastic Gaussian model used by $\I_{\text{opt}}$.
Since we do not have access to the ground-truth $Q^{\pi_\theta}$ or $V^{\pi_\theta}$ in 9x9 Go,
we use one-step TD-style residuals as a proxy for evaluation error.
We take a single set of DNNs $\pi_\theta,Q_\phi,V_\psi$ and 8,192 positions drawn without
replacement from the opening book. For every legal action we compute the one-step bootstrap
target $T(s,a)=r(s,a)+\gamma V_\psi(s')$ and measure the signed residuals $e_V,e_Q$ as follows:
\begin{align}
    e_V(s) \;=\; \textstyle\sum_a \pi(a|s)\,T(s,a) \;-\; V_\psi(s),
    \qquad
    e_Q(s,a) \;=\; T(s,a) \;-\; Q_\phi(s,a),
\end{align}
one per state and one per action at the state, respectively.
The results are reported in Figure \ref{fig:evaluation_error}. 

On the left, we plot the variance, computed as $\tfrac{1}{|S|}\sum_{(s,a) \in S} e_Q(s,a)^2 - \big(\tfrac{1}{|S|}\sum_{(s,a) \in S} e_Q(s,a)\big)^2$ over each action subset $S$, and the same over states for $e_V$.
As expected, the variance is larger for $Q_\phi$ than for $V_\psi$, and increases as the degree of search and the quality (and weight) of the targets provided to the $Q_\phi$ critic reduces.

On the right, we center and scale $e_Q$ per position and compare the standardized residual
distribution to $\N(0,1)$. The bulk is approximately Gaussian, with heavier tails, supporting the
heteroscedastic Gaussian model as a useful approximation.

\begin{figure}[H]
    \centering
    \includegraphics[width=1\linewidth]{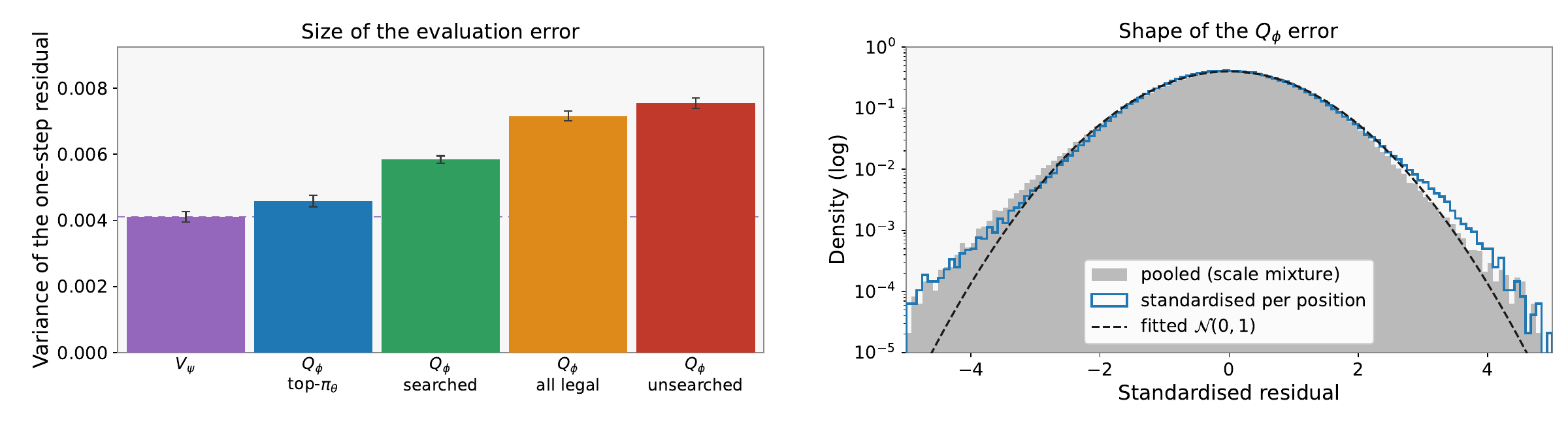}
    \vspace{-0.5cm}
    \caption{
    \textbf{Left:} Variance based on one-step TD residual proxies for evaluation error on different opening-book positions under one network on $9\times9$ Go with 95\% percentile bootstrap over positions.
    \textbf{Right:} the shape of the $Q_\phi$ residual, on a log density scale. 
    In grey the raw pooled residual, a scale mixture over actions of differing $\sigma$.
    In blue is the same data centered and scaled per position; dashed is the fitted $\N(0,1)$. 
    }
    \label{fig:evaluation_error}
\end{figure}

\subsection{Learning curves on model-free online continuous control (SAC on DMC)}
\label{app:additional_results_SAC}

\begin{figure}[H]
    \vspace{-0.3cm}
    \centering
    \includegraphics[width=1\linewidth]{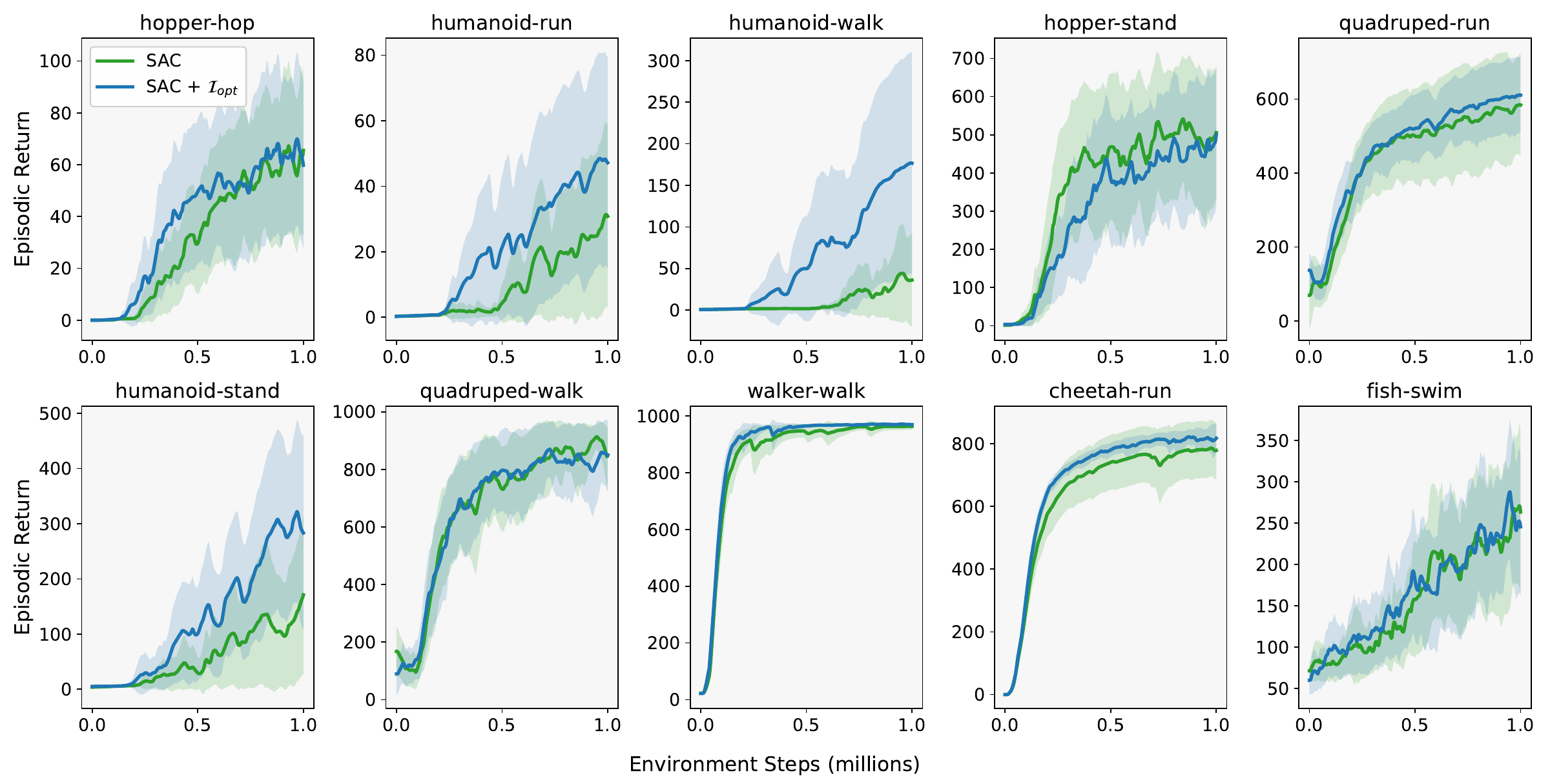}
    \vspace{-0.7cm}
    \caption{
    Per-environment learning curves on the 10 DeepMind Control Suite tasks, showing unnormalised episodic return against environment steps.
    Mean across 10 seeds and $95\%$ interval.
    }
    \label{fig:sac_per_env}
    \vspace{-0.4cm}
\end{figure}

\subsection{Offline RL: per-environment results}
\label{app:additional_results:offline}
We report the final performance per environment in Table \ref{tab:offline_per_env}.
We include results with hyperparameters optimized per datasets, as is chosen by the baseline, as well as with one set of hyperparameters across all datasets, to demonstrate that the $+\I_{\text{opt}}$ algorithm does not appear to be sensitive to a unique per-environment configuration.

\begin{table}[H]
    \centering
    \caption{Per-environment offline RL results on the four D4RL medium-replay datasets.
    Entries are the D4RL score of the final policy evaluated over $56$ episodes, mean and
    $95\%$ Gaussian CI over $20$ runs per arm per environment ($10$ for the single-pair
    setting on halfcheetah). Bold marks the highest mean in each row together with every
    entry that is not significantly worse than it, at $p < 0.05$ under a two-sided Welch
    $t$-test over runs \citep{welch1947generalization}; several bold entries in a row
    therefore indicate a tie.}
    \label{tab:offline_per_env}
    \begin{tabular}{lccc}
        \toprule
        & \multicolumn{2}{c}{\textbf{ReBRAC} $+\,\I_{\text{opt}}$} & \\
        \cmidrule(lr){2-3}
        \textbf{Dataset}
            & per-dataset $\tau^2, \beta$
            & single pair $\tau^2 = 2$, $\beta = 0.5$
            & \textbf{ReBRAC} \\
        \midrule
        halfcheetah-medium-replay & $\mathbf{50.86 \pm 1.09}$ & $\mathbf{50.30 \pm 1.43}$ & $49.65 \pm 0.28$          \\
        hopper-medium-replay      & $\mathbf{87.00 \pm 8.80}$ & $\mathbf{87.00 \pm 8.80}$ & $\mathbf{88.24 \pm 6.80}$ \\
        walker2d-medium-replay    & $\mathbf{84.69 \pm 3.19}$ & $\mathbf{84.69 \pm 3.19}$ & $78.75 \pm 3.91$          \\
        ant-medium-replay         & $\mathbf{40.17 \pm 7.89}$ & $\mathbf{40.17 \pm 7.89}$ & $17.09 \pm 3.05$          \\
        \bottomrule
    \end{tabular}
\end{table}

\section{Hyperparameters}
\label{app:hyperparams}

\subsection{Hyperparameter tuning}
\label{app:hp_tuning}
To keep the comparisons as meaningful as possible, we use each baseline’s pre-tuned hyperparameters and network architecture unchanged for both the baseline and its $+\I_{\text{opt}}$ variant, apart from the additional $\I_{\text{opt}}$ hyperparameters. This controls the comparison by changing only the operator. 
It also leaves open the possibility that the $+\I_{\text{opt}}$ variants could perform better if the baseline hyperparameters were retuned jointly with \(\tau^2\) and \(\beta\).

As a result, outside of the policy iteration experiments, where all hyperparameters were tuned from scratch, the only hyperparameters that were tuned in this work are $\I_{\text{opt}}$'s $\tau^2, \beta$.
To further adhere to each algorithm's standard tuning conditions $\tau^2, \beta$ are tuned per environment for algorithms which are tuned per-environment (ReBRAC), and one set of hyperparameters $\tau^2, \beta$ is tuned and used for an entire domain for the algorithms which are tuned in this way (all the rest: SAC, GumbelAlphaZero and the Policy Iteration experiments).

In the Policy Iteration experiments, all hyperparameters for all algorithms were tuned with an extensive grid search across the 20 environments and the single set of hyperparameters that performed best was chosen. 
The experiment was then repeated for 100 seeds in all environments and only the repeated seeds are shown.

GumbelAlphaZero uses the same set of GumbelMCTS hyperparameters across different board games, for which it was tuned.
Following this choice, we have used one set of hyperparameters in all environments.
To tune $\tau^2$ and $\beta$ we have chosen parameters that induce growth $ \frac{1}{\eta(s,a)} $ with $N(s,a),$ which approximately match the growth of the baseline. 
With this starting point, we've then chosen a set of hyperparameters $\tau^2, \beta$ that result in roughly matching growth within the range of popular evaluation resources used for MCTS $ M = (8, 200),$ from below.
Our logic was that as $\frac{1}{\sigma^2(s,a)}$ increases independently per action, rather than for all actions at the same time, the growth rate should be lower than that of $\frac{1}{\eta(s)}.$
We've then conducted an elo evaluation sweep across hyperparameter combinations around these values with $M = 64$ in 9x9 Go, chosen one set of uniformly well performing hyperparameter, and repeated all experiments with them.

The SAC baseline uses one set of hyperparameters across all environments.
In accordance with this choice, we've done a sweep across hyperparameter combinations for $\I_{\text{opt}}$ across a subset of the environments.
The single set of dominating hyperparameters was chosen and then the experiment was repeated for 10 seeds in all environments.

In the offline RL experiments, the baseline's hyperparameters were tuned per environment.
Following the baselines choice, 
$\tau^2, \beta$ were tuned using a sweep across hyperparameter combinations per environment.
The best performing hyperparameter combination was chosen per environment and the experiments were repeated for 20 additional seeds for this set of hyperparameters. 
However, we have also found that the same set of hyperparameters performed comparably to a unique set of hyperparameters per environment, and we present these results as well in Table \ref{tab:offline_per_env} in Appendix \ref{app:additional_results:offline}.

Table \ref{tab:runs} collects how many independent runs stand behind every number in Figure
\ref{fig:flagship}.

\begin{clblock} 
\begin{table}[H]
\centering
\caption{Number of seeds included in each experiment.}
\label{tab:runs}
\begin{tabular}{ll}
\toprule
\textbf{Experiment} & \textbf{Number of Seeds per Agent} \\
\midrule
Policy iteration & $100$ seeds $\times$ $20$ environments \\
\addlinespace
Model-based, discrete & $5$ training seeds  \\
\addlinespace
Model-free, continuous & $10$ seeds $\times$ $10$ environments  \\
\addlinespace
Offline & $20$ seeds $\times$ $4$ environments \\
\bottomrule
\end{tabular}
\end{table}
\end{clblock}

\subsection{Hyperparameter tables}

\begin{table}[H]
    \centering
    \caption{Hyperparameters for the Policy Iteration experiments.}
    \label{tab:hyperparameters}
    \begin{tabular}{llc}
        \toprule
        \textbf{Component} & \textbf{Parameter} & \textbf{Value} \\
        \midrule
        $\I_{\text{opt}}$ & Advantage prior variance ($\tau^2$) & $(0.17)^2$ \\
                          & KL regularization strength ($\beta$)    & $0.07$ \\
        \addlinespace
        $\I_{\text{reg}}$ & Temperature ($\eta$)                      & $4.89$ \\
        \addlinespace
        Environments & Grid size ($|\S| \times |\A|$)       & $25 \times 5$ \\
        \addlinespace
        Noise schedule & Initial noise bound ($a_0$)        & $50$ \\
        (Eq.\ \ref{eq:dp_noise}) & Decay ($\gamma_{\text{dec}}$) & $0.998$ \\
                          & Variance width ($b$)                 & $5$ \\
        \bottomrule
    \end{tabular}
\end{table}

\begin{table}[H]
    \centering
    \caption{Hyperparameters for the discrete-action model-based GAZ experiments.}
    \label{tab:hp_mcts}
    \begin{tabular}{llc}
        \toprule
        \textbf{Component} & \textbf{Parameter} & \textbf{Value} \\
        \midrule
        $\I_{\text{opt}}$ & Advantage prior variance ($\tau^2$)      & $0.08$ \\
                          & KL regularization strength ($\beta$)                      & $0.01$ \\
        \addlinespace
        $\I_{reg}$ (GumbelMCTS) & $c_{visit}$                        & $50$ \\
                          & $c_{scale}$                              & $1.0$ \\
        \addlinespace
        PUCT (AlphaZero) & $c_{\text{init}}$               & $1.25$ \\
                          & $c_{\text{base}}$                   & $19652$ \\
                          & Dirichlet root noise fraction                 & $0.25$ \\
                          & Dirichlet root noise concentration ($\alpha$) & $0.3$ \\
        \addlinespace
        Search            & Budget $M$, training runs                  & $64$ \\
                          & Budget $M$, fixed-network sweep            & $8, 24, 64, 128, 256$ \\
                          & Num. searched actions at the root ($K$)     & $16$ (unless otherwise specified) \\
                          & Particles ($N$)                          & $1$ \\
                          & Root Gumbel scale                        & $0.0$ \\
                          & Discount ($\gamma$)                      & $0.997$ \\
        \addlinespace
        $Q_\phi$ head & Loss weight ($\lambda_q$) & $0.1$ \\
                          & Bootstrapped actions ($L$)          & $16$ \\
        \addlinespace
        Network      & Residual blocks                & $6$ \\
                          & Channels                       & $128$ \\
                          & Convolution kernel             & $3 \times 3$, stride $1$ \\
                          & Normalisation / activation     & batch norm / ReLU \\
        \addlinespace
        Training          & Self-play frames per iteration      & $262{,}144$ \\
                          & \quad ($1024$ games $\times$ $256$ steps) & \\
                          & Sampled opening plies          & $8$ \\
                          & Training minibatch             & $256$ \\
        \addlinespace
        Optimiser    & Algorithm                      & AdamW \\
                          & Learning rate                  & $10^{-3}$ (constant) \\
                          & First-moment decay             & $0.9$ \\
                          & Other Adam settings            & \texttt{optax} defaults \\
        \bottomrule
    \end{tabular}
\end{table}

$K$ (the number of actions searched by SH, see \citep{gumbel_mz}) is capped at $16$; at the two smallest budgets sequential halving's divisibility
constraints force it lower, giving $K = 4$ at $M = 8$ and $K = 8$ at $M = 24$. The root Gumbel
scale is $0$ following Table 2 of \citet{gumbel_mz}, whose board-game evaluations use $0$; the
value of $1$ in their Table 4 applies to training self-play rather than evaluation.

The body follows \citet{AlphaZero}: a batch-normalised, rectified $3\times3$ convolution,
then $6$ residual blocks of two batch-normalised, rectified $3\times3$ convolutions with a
skip connection, all at $128$ channels. The value head is a $1\times1$ convolution to one
channel, batch norm and ReLU, a linear layer of width $128$, and a $\tanh$ linear layer to a
scalar. The policy and $Q_\phi$ heads have the same shape as each other and differ by game: on
$9\times9$ Go a $3\times3$ convolution, batch norm and ReLU, then a linear layer to the $82$
actions; on chess a $3\times3$ convolution, batch norm and ReLU, then a
second $3\times3$ convolution to the $73$ plane-encoded actions. $Q_\phi$ is $\tanh$, one value
per action, and is built in every arm but trained only where a $Q$ head is used. Inputs are
\texttt{pgx}'s encodings: $17$ planes for $9\times9$ Go and $119$ for chess.

\begin{table}[H]
    \centering
    \caption{Hyperparameters for the online classical continuous control SAC experiments}
    \label{tab:hp_sac}
    \begin{tabular}{llc}
        \toprule
        \textbf{Component} & \textbf{Parameter} & \textbf{Value} \\
        \midrule
        $\I_{\text{opt}}$ & Advantage prior variance ($\tau^2$) & $0.4$ \\
                          & Num. actions sampled for $V_{\phi_i}(s)$ & $32$ \\
        \addlinespace
        SAC          & Optimizer                              & Adam \\
                          & Actor learning rate                    & $3\cdot10^{-4}$ \\
                          & Critic and entropy learning rate       & $10^{-3}$ \\
                          & Batch size                             & $256$ \\
                          & Replay buffer size                     & $10^{6}$ \\
                          & Steps before learning starts           & $10^{4}$ \\
                          & Discount ($\gamma$)                    & $0.99$ \\
                          & Target smoothing coefficient           & $0.005$ \\
                          & Actor update interval                  & $2$ \\
                          & Target network update interval         & $1$ \\
                          & Entropy coefficient ($\alpha$)         & autotuned \\
                          & Target entropy                         & $-\dim(\mathcal A)$ \\
                          & Critics                                & $2$ \\
        \addlinespace
        Network      & Actor trunk                            & $2\times256$, ReLU \\
                          & Actor heads                            & mean and $\log\sigma$ \\
                          & $\log\sigma$ range                     & $[-5, 2]$ \\
                          & Action squashing                       & $\tanh$ \\
                          & Critic trunk (input $[s,a]$)           & $2\times256$, ReLU \\
                          & Normalisation                          & none \\
                          & Initialisation                         & PyTorch defaults \\
        \bottomrule
    \end{tabular}
\end{table}

\begin{table}[H]
    \centering
    \caption{Hyperparameters for the offline RL ReBRAC experiments.}
    \label{tab:hp_offline}
    \small
    \begin{tabular}{llc}
        \toprule
        \textbf{Component} & \textbf{Parameter} & \textbf{Value} \\
        \midrule
        $\I_{\text{opt}}$ & Advantage prior variance ($\tau^2$), halfcheetah-medium-replay & $8$ \\
                          & Advantage prior variance ($\tau^2$), hopper-medium-replay      & $2$ \\
                          & Advantage prior variance ($\tau^2$), walker2d-medium-replay    & $2$ \\
                          & Advantage prior variance ($\tau^2$), ant-medium-replay         & $2$ \\
                          & KL regularization strength ($\beta$), all datasets                              & $0.5$ \\
        \addlinespace
        ReBRAC            & Critics (\texttt{num\_critics})     & $2$ \\
                          & Optimizer                                                 & Adam ($\epsilon = 10^{-5}$) \\
                          & Learning rate (actor and critics)                         & $10^{-3}$ \\
                          & Batch size                                                & $1024$ \\
                          & Gradient steps                                            & $10^{6}$ \\
                          & Discount ($\gamma$)                                       & $0.99$ \\
                          & Target smoothing coefficient                              & $0.005$ \\
                          & Target policy noise                                       & $0.2$ \\
                          & Target policy noise clip                                  & $0.5$ \\
                          & Critic updates per actor update                           & $2$ \\
        \addlinespace
        Network      & Actor trunk                                               & $3\times256$, ReLU \\
                          & Actor output                                              & $\dim(\mathcal A)$, $\tanh$ \\
                          & Critic trunk (input $[s,a]$)                              & $3\times256$, ReLU, LayerNorm \\
                          & Critic weight sharing                                     & none \\
                          & Observation normalisation                                 & no \\
                          & Hidden bias initialisation                                & $0.1$ \\
                          & Output initialisation, actor                              & $\mathcal U(\pm10^{-3})$ \\
                          & Output initialisation, critics                            & $\mathcal U(\pm3\cdot10^{-3})$ \\
                          & Target networks                                           & actor and critics \\
        \addlinespace
        Actor BC     & $\beta_1$, halfcheetah-medium-replay                       & $0.01$ \\
        coefficient  & $\beta_1$, hopper-medium-replay                           & $0.05$ \\
                          & $\beta_1$, walker2d-medium-replay                         & $0.05$ \\
                          & $\beta_1$, ant-medium-replay                              & $0.001$ \\
        \addlinespace
        Critic BC    & $\beta_2$, halfcheetah-medium-replay                       & $0.001$ \\
        coefficient  & $\beta_2$, hopper-medium-replay                           & $0.5$ \\
                          & $\beta_2$, walker2d-medium-replay                         & $0.01$ \\
                          & $\beta_2$, ant-medium-replay                              & $0.01$ \\
        \bottomrule
    \end{tabular}
\end{table}

\end{document}